\documentclass{article}

\usepackage{times}
\usepackage{graphicx} 
\usepackage{subfigure} 

\usepackage{natbib}

\usepackage{algorithm}
\usepackage{algorithmic}

\usepackage{url}
\usepackage[utf8]{inputenc}
\usepackage{amsmath}
\usepackage{amssymb}
\usepackage{amsthm}
\usepackage{bbm}
\usepackage{bm}
\usepackage{color}
\usepackage{dsfont}
\usepackage{multicol}
\usepackage{natbib}
\usepackage{subfigure}
\usepackage{times}
\usepackage{amsfonts}
\usepackage{footmisc}

\usepackage{hyperref}

\usepackage[accepted]{icml2015}

\newcommand{\oracle}{{\tt ORACLE}}
\newcommand{\linucb}{{\tt LinUCB}}
\newcommand{\combts}{{\tt CombLinTS}}
\newcommand{\comblints}{{\tt CombLinTS}}
\newcommand{\comblinucb}{{\tt CombLinUCB}}

\newtheorem{theorem}{Theorem}

\newtheorem{lemma}{Lemma}

\newtheorem{remark}{Remark}

\newcommand{\bw}{{\bf w}}

\newcommand{\cA}{\mathcal{A}}

\newcommand{\cH}{\mathcal{H}}

\newcommand{\realset}{\mathbb{R}}

\newcommand{\abs}[1]{\left|#1\right|}

\newcommand{\E}[2]{\mathbb{E}_{#2} \! \left[#1\right]}

\newcommand{\I}[1]{\mathds{1} \! \left\{#1\right\}}

\DeclareMathOperator*{\argmax}{arg\,max\,}
\DeclareMathOperator*{\argmin}{arg\,min\,}

\icmltitlerunning{Efficient Learning in Large-Scale Combinatorial Semi-Bandits}

\begin{document} 


\twocolumn[
\icmltitle{Efficient Learning in Large-Scale Combinatorial Semi-Bandits}

\icmlauthor{Zheng Wen}{zhengwen@yahoo-inc.com}
\icmladdress{Yahoo Labs, Sunnyvale, CA
            }
\icmlauthor{Branislav Kveton}{kveton@adobe.com}
\icmladdress{Adobe Research, San Jose, CA}
\icmlauthor{Azin Ashkan}{azin.ashkan@technicolor.com}
\icmladdress{Technicolor Labs, Los Altos, CA
}

\icmlkeywords{}

\vskip 0.3in
]

\begin{abstract} 
A stochastic combinatorial semi-bandit is an online learning problem where at each step a 
learning agent chooses a subset of ground items subject to  combinatorial constraints, and then 
observes stochastic weights of these items and receives their sum as a payoff.
In this paper, we consider efficient learning in large-scale combinatorial semi-bandits with linear generalization, and as a solution, propose two learning algorithms
called \emph{Combinatorial Linear Thompson Sampling ($\combts$)} and \emph{Combinatorial Linear UCB ($\comblinucb$)}.
Both algorithms are computationally efficient as long as  the offline version of the combinatorial problem can be solved efficiently.
We establish that $\combts$ and $\comblinucb$ are also provably statistically efficient under reasonable assumptions, by developing regret bounds that are independent of the problem scale (number of items) and sublinear in time. We also evaluate $\combts$ on a variety of 
problems with thousands of items.
Our experiment results demonstrate that $\combts$ is scalable, robust to the choice of algorithm parameters, and significantly outperforms the best of our baselines.
\end{abstract}


\section{Introduction}
\label{sec:introduction}

Combinatorial optimization is a mature field \citep{papadimitriou98combinatorial}, which has countless practical applications. One of the most studied problems in combinatorial optimization is maximization of a modular function subject to combinatorial constraints. Many important problems, such as minimum spanning tree (MST), shortest path, and maximum-weight bipartite matching, can be viewed as instances of this problem.

In practice, the optimized modular function is often unknown and needs to be learned while repeatedly solving the problem. This class of learning problems was recently formulated as a combinatorial bandit/semi-bandit, depending on the feedback model \citep{audibert14regret}. Since then, many combinatorial bandit/semi-bandit algorithms have been proposed: for the stochastic setting \citep{gai12combinatorial,chen13combinatorial,russo2014posteriorinformation,kveton15combinatorial}; for the adversarial setting \citep{cesabianchi12combinatorial,audibert14regret,neu2013efficient}; and for subclasses of combinatorial problems, matroid and polymatroid bandits \citep{kveton14matroid,kveton14polymatroid}, submodular maximization \citep{wen13sequential,gabillon13adaptive}, and cascading bandits \citep{kveton2015cascading}.
Many regret bounds have been established for the combinatorial semi-bandit algorithms. 
To achieve an $O (\sqrt{n})$ dependence on time $n$, all of the regret bounds are $\Omega(\sqrt{L})$, where $L$ is the number of items. The dependence on $L$ is intrinsic because the algorithms estimate the weight of each item separately, and matching lower bounds have been established (Section~\ref{sec:linear_generalization}).

However, in many real-world problems, the number of items $L$ is intractably large. For instance, online advertising 
in a mainstream commercial website can be 
viewed as a bipartite matching problem with millions of users and products; routing in the Internet can be formulated 
as a shortest path problem with billions of edges.
Thus, learning algorithms with $\Omega(\sqrt{L})$ regret are impractical in such problems.
On the other hand, in many problems, items have features and their weights are similar when the features are similar. In movie recommendation, for instance, the expected ratings of movies that are close in the latent space are also similar. In this work, we show how to leverage this structure to learn to make good decisions more efficiently. More specifically, we assume a \emph{linear generalization} across the items: conditioned on the features of an item, the expected weight of that item can be estimated using a linear model. Our goal is to develop more efficient learning algorithms for combinatorial semi-bandits with linear generalization.

It is relatively easy to extend many linear bandit algorithms, such as Thompson sampling \citep{thompson1933likelihood, shipra2012, russo2013posterior} and Linear UCB ($\linucb$, see \citet{auer02using, dani08stochastic, abbasi-yadkori11improved}) , to combinatorial semi-bandits with linear generalization. In this paper, we propose two learning algorithms, Combinatorial Linear Thompson Sampling ($\comblints$) and Combinatorial Linear UCB ($\comblinucb$), based on Thompson sampling and $\linucb$. Both $\comblints$ and $\comblinucb$ are computationally efficient, as long as the offline version of the combinatorial problem can be solved efficiently. The first major contribution of the paper is that we establish a \emph{Bayes regret bound} on $\comblints$ and a \emph{regret bound} on $\comblinucb$, under reasonable assumptions. Both bounds are $L$-independent, and sublinear in time. The second major contribution of the paper is that we evaluate $\comblints$ on a variety of 
problems with thousands of items, and two of these problems are based on real-world datasets.
We only evaluate $\comblints$ since recent literature \citep{chapelle2011empirical} suggests that Thompson sampling algorithms usually outperform UCB-like algorithms in practice. Our experimental results demonstrate that $\comblints$ is scalable, robust to the choice of algorithm parameters, and significantly outperforms the best of our baselines.
It is worth mentioning that our derived $L$-independent regret bounds also hold in cases with $L=\infty$.
Moreover, 
as we will discuss in Section \ref{sec:conclusion},
our proposed algorithms and their analyses can be easily extended to the \emph{contextual combinatorial semi-bandits}.

Finally, we briefly review some relevant papers. \citet{gabillon14largeScale} and \citet{yue11linear} focus on submodular maximization
with linear generalization. Our paper differs from these two papers
 in the following two aspects: (1) our paper allows general combinatorial constraints while 
 they do not; (2) our paper focuses on maximization of modular functions while 
 they
focus on submodular maximization.



\section{Combinatorial Optimization}
\label{sec:combinatorial}
We focus on a class of combinatorial optimization problems that aim to find a \emph{maximum-weight} set 
from a given family of sets. Specifically, one such combinatorial optimization problem can be represented as 
a triple $\left(E, \cA, \bw \right)$, where (1) $E =\{ 1, \ldots, L \}$ is a set of $L$ items, called the \emph{ground set}, (2)
$\cA \subseteq \left \{ A \subseteq E : \, |A| \leq K \right \}$
is a family of subsets of $E$ with up to $K$ items, where $K \leq L$, and (3) $\bw: E \rightarrow \realset$ is a \emph{weight function} that assigns each item $e$ in the ground set $E$ a real number. The total weight of all items in a set $A \subseteq E$ is defined as: 
\begin{align}
f(A, \bw)=\textstyle \sum_{e \in A} \bw(e),  \label{eq:reward}
\end{align}
which is a linear functional of $\bw$ and a modular function in $A$. A set $A^\mathrm{opt}$ is a maximum-weight set in $\cA$ if: 
\begin{align}
A^{\mathrm{opt}} \in \argmax_{A \in \cA} f(A, \bw) = \argmax_{A \in \cA} \textstyle \sum_{e \in A} \bw(e).  \label{eq:CombinatorialOptimization}
\end{align}
Many classical combinatorial optimization problems, such as finding an MST, bipartite matching, 
the shortest path problem and the traveling salesman problem (TSP), have form (\ref{eq:CombinatorialOptimization}). Though some of these problems can be solved efficiently (e.g. bipartite matching), others (e.g. TSP) are known to be NP-hard. However, for many such NP-hard problems, there exist computationally efficient \emph{approximation algorithms} and/or \emph{randomized algorithms} that achieve near-optimal solutions with high probability. Similarly to \citet{chen13combinatorial}, in this paper, we allow the agent to use any approximation / randomized algorithm $\oracle$ to solve (\ref{eq:CombinatorialOptimization}), and denote its solution as $A^\ast = \oracle(E, \cA, \bw)$. To distinguish from a learning algorithm, we refer to a combinatorial optimization algorithm as an \emph{oracle} in this paper.


\section{Combinatorial Semi-Bandits with Linear Generalization}
\label{sec:linearCB}

Many real-world problems are combinatorial in nature. In recommender systems, for instance, the user is typically recommended $K$ items out of $L$. The value of an item, such as the expected rating of a movie, is never known perfectly and has to be refined while repeatedly recommending to 
the pool of the users. Recommender problems are known to be highly structured. In particular, it is well known that the user-item matrix is typically low-rank \cite{koren09matrix} and that the value of an item can be written as a linear combination of its position in the latent space. In this work, we propose a learning algorithm for combinatorial optimization that leverages this structure. In particular, we assume that the weight of each item is a linear function of its features and then we learn the parameters of this model, jointly for all items.

\subsection{Combinatorial Semi-Bandits}
\label{sec:return}

We formalize our learning problem as a combinatorial semi-bandit. A combinatorial semi-bandit is a triple $(E, \cA, P)$, where $E$ and $\cA$ are defined in Section \ref{sec:combinatorial} and $P$ is a probability distribution over the weights $\bw \in \realset^L$ of the items in the ground set $E$. We assume that the weights $\bw$ are drawn i.i.d. from $P$. The mean weight is denoted by $\bar{\bw}=\E{\bw}{}$. Each item $e$ is associated with an \emph{arm} and we assume that \emph{multiple arms} can be pulled. A subset of arms $A \subseteq E$ can be pulled if and only if $A \in \cA$. The return of pulling arms $A$ is $f(A, \bw)$ (Equation (\ref{eq:reward})), the sum of the weights of all items in $A$. After the arms $A$ are pulled, we observe the individual return of each arm, $\{\bw(e): \, e \in A \}$. This feedback model is known as \emph{semi-bandit} \cite{audibert14regret}.

We assume that the combinatorial structure $(E, \cA)$ is known and the distribution $P$ is unknown. We would like to stress that we do not make any structural assumptions on $P$.
The optimal solution to our problem is a maximum-weight set in expectation:
\begin{align}
\hspace{-0.3cm}
A^{\mathrm{opt}} \in  \argmax_{A \in \cA} \E{f(A, \bw)}{\bw} 
= 
\argmax_{A \in \cA} \sum_{e \in A} \bar{\bw}(e).  \label{eq:A_opt}
\end{align}
This objective is equivalent to the one in Equation (\ref{eq:CombinatorialOptimization}). 

Our learning problem is episodic. In each episode $t$, the learning agent adaptively chooses $A^t \in \cA$ based on its observations
of the weights up to episode $t$, gains $f(A^t , \bw_t)$, and observes the weights of all chosen items in episode $t$,
$\left \{ (e , \bw_t(e)) : \, e \in A^t \right \}$. The learning agent interacts with the combinatorial semi-bandit for $n$ times and 
its goal is to maximize the expected cumulative return in $n$-episodes $\E{\sum_{t=1}^n f(A^t, \bw_t)}{}$, where the expectation is over (1) the random weights $\bw_t$'s, (2) possible randomization in the learning algorithm, and (3) $\bar{\bw}$ if it is randomly generated. 
Notice that the choice of $A^t$ impacts both the return and observations in episode $t$. So we need to trade off \emph{exploration} and \emph{exploitation}, similarly to other bandit problems.

\subsection{Linear Generalization}
\label{sec:linear_generalization}

As we have discussed in Section \ref{sec:introduction}, many provably efficient algorithms have been developed for various combinatorial semi-bandits of form (\ref{eq:A_opt}) \cite{chen13combinatorial,gai12combinatorial,kveton14matroid,russo2014posteriorinformation}. However, since there are $L$ parameters to learn and these algorithms do not consider \emph{generalization} across items, the derived upper bounds on the expected cumulative regret and/or the Bayes cumulative regret of these algorithms are at least $O(\sqrt{L})$. Furthermore, \citet{audibert14regret} has derived an $\Omega (\sqrt{LKn})$ lower bound on adversarial combinatorial semi-bandits, while \citet{kveton14matroid} has derived an asymptotic $\Omega(L \log(n)/\Delta)$ gap-dependent lower bound on stochastic combinatorial semi-bandits, where $\Delta$ is the ``gap".

However, in many modern combinatorial semi-bandit problems, $L$ tends to be enormous. Thus, an $O(\sqrt{L})$ regret is unacceptably large in these problems. On the other hand, in many practical problems, there exists a \emph{generalization model} based on which the weight of one item can be (approximately) inferred based on the weights of other items.
By exploiting such generalization models, an 
$o (\sqrt{L})$ or even an $L$-independent cumulative regret might be achieved.

In this paper, we assume that there is a (possibly imperfect) linear generalization model across the items. Specifically, we assume that the agent knows a \emph{generalization matrix} $\Phi \in \realset^{L \times d}$ s.t. $\bar{\bw}$ either lies in or is ``close" to the subspace $\mathrm{span} \left[ \Phi \right]$. 
We use $\phi_e$ to denote the transpose of the $e$-th row of $\Phi$, and refer to it as the \emph{feature vector} of item $e$.
Without loss of generality, we assume that 
$\mathrm{rank} \left[ \Phi \right]=d$.

Similar to \citet{wennips13}, we distinguish between the \emph{coherent learning} cases, in which $\bar{\bw} \in \mathrm{span} \left[ \Phi \right]$, and the \emph{agnostic learning} cases, in which $\bar{\bw} \notin \mathrm{span} \left[ \Phi \right]$. Like existing literature on linear bandits \citep{dani08stochastic, abbasi-yadkori11improved}, the analysis in this paper focuses on coherent learning cases. However, we would like to emphasize that both of our proposed algorithms, $\comblints$ and $\comblinucb$, are also applicable to the agnostic learning cases. As is demonstrated in Section \ref{sec:experiments}, $\comblints$ performs well in the agnostic learning cases.

Finally, we define $\theta^\ast = \argmin_{\theta} \| \bar{\bw} - \Phi \theta \|_2 $. Since $\mathrm{rank} \left[ \Phi \right]=d$, $\theta^\ast$ is uniquely defined. Moreover, in coherent learning cases, we have $\bar{\bw}=\Phi \theta^*$.

\subsection{Performance Metrics}
Let $A^\ast = \oracle (E, \cA, \bar{\bw})$. In this paper, we measure the performance loss of a learning algorithm with respect to $A^\ast$.
Recall that the learning algorithm chooses $A^t$ in episode $t$,
we define $R_t=f(A^\ast, \bw_t)-f(A^t, \bw_t)$ as the \emph{realized
regret} in episode $t$. If the expected weight $\bar{\bw}$ is fixed but unknown, we define the \emph{expected cumulative regret} of the learning algorithm in $n$ episodes as
\begin{align}
R(n) = \textstyle \sum_{t=1}^n \mathbb{E} \left[  R_t \middle| \bar{\bw} \right], \label{eq:regret}
\end{align}
where the expectation is over random weights and possible randomization in the learning algorithm. If necessary, we denote $R(n)$ as $R(n ; \bar{\bw})$ to emphasize the dependence on $\bar{\bw}$.
On the other hand, if $\bar{\bw}$ is randomly generated or the agent has a prior belief in $\bar{\bw}$, then from \citet{russo2013posterior}, the \emph{Bayes cumulative regret} of the learning algorithm in $n$ episodes is defined as
\begin{align}
R_{\mathrm{Bayes}}(n)=\E{R(n ; \bar{\bw})}{\bar{\bw}}= \textstyle \sum_{t=1}^n \mathbb{E} \left[  R_t \right],
\label{eq:Bayes_regret}
\end{align}
where the expectation is also over $\bar{\bw}$. That is, $R_{\mathrm{Bayes}}(n)$ is a weighted average of $R(n ; \bar{\bw})$ under the prior on $\bar{\bw}$.

\section{Learning Algorithms}
\label{sec:algorithms}

In this section, we propose two learning algorithms for combinatorial semi-bandits: Combinatorial Linear Thompson Sampling ($\comblints$)
and Combinatorial Linear UCB ($\comblinucb$), which are respectively motivated by Thompson sampling and $\linucb$.
Both algorithms maintain a mean vector $\bar{\theta}_t$ and a covariance matrix
$\Sigma_t$, and use Kalman filtering to update  $\bar{\theta}_t$ and $\Sigma_t$.
They differ in how to choose $A^t$ (i.e. how to explore) in each episode $t$: $\comblints$ chooses $A^t$ based on a randomly sampled coefficient vector
$\theta_t$, while $\comblinucb$ chooses $A^t$ based on the \emph{optimism in the face of uncertainty (OFU)} principle.

\subsection{Combinatorial Linear Thompson Sampling}
\label{sec:comblints}
The psuedocode of $\comblints$ is given in Algorithm \ref{algorithm:combts}, where $(E, \cA)$ is the combinatorial structure, $\Phi$ is the generalization matrix, $\oracle$ is a combinatorial optimization algorithm, and 
$\lambda$ and $\sigma$ are two algorithm parameters controlling the \emph{learning rate}.
Specifically, $\lambda$ is an \emph{inverse-regularization} parameter and smaller $\lambda$ makes the covariance matrix 
$\Sigma_t$ closer to $0$. Thus, a too small $\lambda$ will lead to insufficient exploration and significantly reduce the performance of
$\comblints$.
%
On the other hand, $\sigma$ controls the decrease rate of the covariance matrix $\Sigma_t$.
In particular, 
a large $\sigma$ will lead to slow learning, while a too small $\sigma$ will make the algorithm quickly converge to some sub-optimal coefficient vector.

\begin{algorithm*}[t]
\caption{Compute $\bar{\theta}_{t+1}$ and $\Sigma_{t+1}$ based on Kalman Filtering}
\label{algorithm:kalman}
\begin{algorithmic}
\STATE {\bf Input:} $\bar{\theta}_t$, $\Sigma_t$, $\sigma$, and feature-observation pairs $\left \{ \left( \phi_e ,\bw_t(e) \right) : \, e \in A^t \right \}$
\STATE \vspace{-0.12in}
\STATE Initialize $\bar{\theta}_{t+1} \leftarrow \bar{\theta}_t$ and $\Sigma_{t+1} \leftarrow \Sigma_t$
\FOR{$k=1,\ldots,|A^t|$}
    \STATE Update $\bar{\theta}_{t+1}$ and $\Sigma_{t+1}$ as follows, where $a_k^t$ is the $k$th element in $A^t$
    \begin{align}
    \bar{\theta}_{t+1} \leftarrow & \left[ I- \frac{\Sigma_{t+1} \phi_{a_k^t} \phi_{a_k^t}^T}{\phi_{a_k^t}^T \Sigma_{t+1} \phi_{a_k^t}+ \sigma^2} \right] \bar{\theta}_{t+1} + 
    \left[ \frac{\Sigma_{t+1} \phi_{a_k^t} }{\phi_{a_k^t}^T \Sigma_{t+1} \phi_{a_k^t}+ \sigma^2} \right]\bw_t\left(a_k^t \right) \nonumber \\
    \Sigma_{t+1} \leftarrow &
    \Sigma_{t+1}- \frac{\Sigma_{t+1} \phi_{a_k^t} \phi_{a_k^t}^T \Sigma_{t+1}}{\phi_{a_k^t}^T \Sigma_{t+1} \phi_{a_k^t}+ \sigma^2},
    \end{align}
\ENDFOR
\STATE \vspace{-0.12in}
\STATE {\bf Output:} $\bar{\theta}_{t+1}$ and $\Sigma_{t+1}$
\end{algorithmic}
\end{algorithm*}

\begin{algorithm}[t]
  \caption{Combinatorial Linear Thompson Sampling 
  }
  \label{algorithm:combts}
  \begin{algorithmic}
    \STATE {\bf Input:} Combinatorial structure $(E, \cA)$, generalization matrix $\Phi \in \realset^{L \times d}$,
    algorithm parameters $\lambda, \sigma>0$, oracle $\oracle$
    \STATE \vspace{-0.12in}
    \STATE Initialize $\Sigma_1 \leftarrow \lambda^2 I \in \realset^{d \times d}$ and $\bar{\theta}_1=0 \in \realset^d$
    \FORALL{$t = 1, 2, \ldots, n$} 
    \STATE Sample $\theta_{t} \sim N \left( \bar{\theta}_t, \Sigma_t \right)$ 
    \STATE Compute $A^t \leftarrow \oracle (E, \cA, \Phi \theta_t )$
    \STATE Choose set $A^t$, and observe $\bw_t(e)$, $\forall e \in A^t$
    \STATE Compute $\bar{\theta}_{t+1}$ and $\Sigma_{t+1}$ based on Algorithm \ref{algorithm:kalman}    
    \ENDFOR
  \end{algorithmic}
\end{algorithm}

In each episode $t$, Algorithm \ref{algorithm:combts} consists of three steps.
First, it randomly samples a coefficient vector $\theta_t$ from a Gaussian distribution. Second, it computes $A^t$ based on $\theta_t$ and the pre-specified oracle. Finally, it updates the mean vector $\bar{\theta}_{t+1}$ and the covariance matrix $\Sigma_{t+1}$ based on Kalman filtering (Algorithm \ref{algorithm:kalman}).

It is worth pointing our that if (1) $\bar{\bw}=\Phi \theta^\ast$, (2) the prior on $\theta^\ast$
is $N(0, \lambda^2 I)$, and (3) $\forall (t,e)$, the noise
$\eta_t(e)=\bw_t(e)-\bar{\bw}(e)$ is independently sampled from $N(0 , \sigma^2)$, then in each episode $t$, the
$\comblints$ algorithm samples $\theta_t$ from the posterior distribution of $\theta^\ast$.
We henceforth refer to a case satisfying condition (1)-(3) as a \emph{coherent Gaussian case}. 
Obviously, the $\comblints$ algorithm can be applied to more general cases, 
even to cases with no prior and/or agnostic learning cases.

\subsection{Combinatorial Linear UCB}

The pseudocode of $\comblinucb$ is given in Algorithm \ref{algorithm:comblinucb}, where $E$, $\cA$, $\Phi$ and $\oracle$ are defined the same as in
Algorithm \ref{algorithm:combts}, and $\lambda$, $\sigma$, and $c$ are three algorithm parameters.
Similarly, $\lambda$ is an inverse-regularization parameter, $\sigma$ 
controls the decrease rate of the covariance matrix,
and 
$c$ controls the \emph{degree of optimism} (exploration). Specifically, if $c$ is too small, the algorithm might converge to some sub-optimal coefficient
vector due to insufficient exploration; on the other hand, too large $c$ will lead to excessive exploration and slow learning.

\begin{algorithm}[t]
  \caption{Combinatorial Linear UCB}
  \label{algorithm:comblinucb}
  \begin{algorithmic}
    \STATE {\bf Input:} Combinatorial structure $(E, \cA)$, generalization matrix $\Phi \in \realset^{L \times d}$,
    algorithm parameters $\lambda, \sigma, c >0$, oracle $\oracle$
     \STATE \vspace{-0.12in}
    \STATE Initialize $\Sigma_1 \leftarrow \lambda^2 I \in \realset^{d \times d}$ and $\bar{\theta}_1=0 \in \realset^d$
    \FORALL{$t = 1, 2, \ldots, n$} 
    \STATE Define the UCB weight vector $\hat{\bw}_t$ as
    \[
    \hat{\bw}_t (e) = \left < \phi_e, \bar{\theta}_t \right > + c \sqrt{\phi_e^T \Sigma_t \phi_e } \quad \forall e \in E
    \]
    \STATE Compute $A^t \leftarrow \oracle (E, \cA, \hat{\bw}_t )$
    \STATE Choose set $A^t$, and observe $\bw_t(e)$, $\forall e \in A^t$
    \STATE Compute $\bar{\theta}_{t+1}$ and $\Sigma_{t+1}$ based on Algorithm \ref{algorithm:kalman}  
    \ENDFOR
  \end{algorithmic}
\end{algorithm}

In each episode $t$, Algorithm \ref{algorithm:comblinucb} also consists of three steps.
First, for each $e \in E$, it computes an upper confidence bound (UCB) $\hat{\bw}_t (e)$.
Second, it computes $A^t$ based on $\hat{\bw}_t$ and the pre-specified oracle. Finally, it updates $\bar{\theta}_{t+1}$ and $\Sigma_{t+1}$ based on Kalman filtering (Algorithm \ref{algorithm:kalman}).


\section{Regret Bounds}
\label{sec:bounds}
In this section, we present a Bayes regret bound on $\comblints$, and a regret bound on
$\comblinucb$. We will also briefly discuss how these bounds are derived, as well as their tightness.
The detailed proofs are left to the appendices.
Without loss of generality, throughout this section, we assume that
$\| \phi_e \|_2 \leq 1$, $\forall e \in E$.

\subsection{Bayes Regret Bound on $\comblints$}
\label{sec:bound_combts}


We have the following upper bound on $R_{\mathrm{Bayes}} (n)$ when $\comblints$ is applied to a coherent Gaussian case with the right parameter. 

\begin{theorem}
\label{thm:bound_combts}
If (1) $\bar{\bw}=\Phi \theta^\ast$, (2) the prior on $\theta^\ast$ is 
$N(0, \lambda^2 I)$, (3) the noises are i.i.d. sampled from $N(0 , \sigma^2)$,
and (4) $\lambda \geq \sigma$,
then under $\comblints$ algorithm with parameter $(\Phi, \lambda, \sigma)$, we have
\begin{align}
\label{eq:bound_1}
R_{\mathrm{Bayes}}(n) \leq \tilde{O} \left(
K \lambda \sqrt{dn \min \left\{\ln(L), d \right\}}
\right) .
\end{align}
\end{theorem}
Notice that condition (1)-(3) ensure it is a coherent Gaussian case, and condition (4) almost always holds\footnote{Condition (4) is not essential, please refer to Theorem \ref{thm:bound_1_refine} in Appendix \ref{proof:combts} for a Bayes regret bound without condition
(4).}.
The $\tilde{O}$ notation hides the logarithm factors.
We also note that Equation (\ref{eq:bound_1}) is a minimum of two bounds. The first bound is $L$-dependent, but it is only $O ( \sqrt{\ln(L)})$; on the other hand,
the second bound is $L$-independent, but is $\tilde{O}(d)$ instead of 
$\tilde{O} ( \sqrt{d})$.


We now outline the proof of Theorem \ref{thm:bound_combts}, which is motivated by
\citet{russo2013posterior} and \citet{dani08stochastic}.
Let $\cH_t$ denote the ``history" (i.e. all the available information) by the start of episode $t$. Note that from the Bayesian perspective, conditioning on $\cH_t$, $\theta^\ast$
and $\theta_t$ are i.i.d. drawn from $N(\bar{\theta}_t, \Sigma_t)$ \citep{russo2013posterior}.
This is because that conditioning on $\cH_t$, the posterior belief in $\theta^\ast$ is $N(\bar{\theta}_t, \Sigma_t)$ and based on Algorithm \ref{algorithm:combts}, $\theta_t$ is independently sampled from $N(\bar{\theta}_t, \Sigma_t)$. Since $\oracle$ is a fixed combinatorial optimization algorithm (even though it can be independently randomized), and $E,\cA, \Phi$ are all fixed, then 
conditioning on $\cH_t$,
$A^\ast$ and $A^t$ are also i.i.d., furthermore, $A^\ast$ is conditionally independent of $\theta_t$, and $A^t$ is conditionally independent of $\theta^\ast$.

To simplify the exposition, $\forall \theta \in \realset^d$ and $\forall A \subseteq E$, we define
\[
g(A , \theta)= \textstyle \sum_{e \in A} \left< \phi_e, \theta \right>,
\]
where $\left< \cdot , \cdot \right>$ is an alternative notation for inner product.
Thus we have
$\E{R_t \middle | \cH_t}{}=\E{g(A^\ast, \theta^\ast) - g(A^t, \theta^\ast) \middle | \cH_t}{}$. We also define a UCB function
$U_t : 2^E \rightarrow \realset$ as
%
\[
U_t (A)=\textstyle \sum_{e \in A} \left[  \left<\phi_e, \bar{\theta}_t \right> + c 
\sqrt{\phi_e^T \Sigma_t \phi_e} \right],
\]
%
where $c>0$ is a constant to be specified. Notice that conditioning on $\cH_t$, $U_t$ is a deterministic function and $A^\ast, A^t$ are i.i.d., then 
$\E{U_t (A^t)- U_t (A^\ast) \middle | \cH_t}{}=0$ and 
%
\begin{align}
\E{ R_t \middle | \cH_t}{} 
=& \E{ g(A^\ast, \theta^\ast)- U_t (A^\ast)  \middle | \cH_t}{} \nonumber \\
+&
\E{U_t (A^t)-g(A^t, \theta^*)  \middle | \cH_t}{}. \label{eq:bayes_decomposition}
\end{align}
%
Theorem \ref{thm:bound_combts} follows by respectively bounding
the two terms on the righthand side of Equation (\ref{eq:bayes_decomposition}).
Two key observations are (1) if $c = \tilde{O} \left( \sqrt{\min \left \{ \ln(L), d \right \}}\right)$,
then 
\[
\E{ g(A^\ast, \theta^\ast)- U_t (A^\ast)  \middle | \cH_t}{} = O(1),
\]
and (2)
\[
\E {U_t (A^t)-g(A^t, \theta^*) \middle | \cH_t}{} =  c \E{ \textstyle \sum_{e \in A^t} 
\sqrt{\phi_e^T \Sigma_t \phi_e} \middle | \cH_t}{},
\]
and we have a worst-case bound (see Lemma \ref{lem:key_ineq_1} in Appendix \ref{proof:combts}) on 
$\sum_{t=1}^n \sum_{e \in A^t} 
\sqrt{\phi_e^T \Sigma_t \phi_e}$.
Please refer to Appendix \ref{proof:combts} for the detailed proof for Theorem \ref{thm:bound_combts}.

Finally, we briefly discuss the tightness of our bound. Without loss of generality, we assume that $\lambda=1$.
For the special case when $\Phi=I$ (i.e. no generalization), 
\citet{russo2014posteriorinformation} provides an $O(\sqrt{LK \log(L/K) n})$ upper bound on $R_{\mathrm{Bayes}}(n)$ when Thompson sampling is applied,
and \citet{audibert14regret} provides an $\Omega (\sqrt{LKn})$ lower bound\footnote{\citet{audibert14regret} focuses on the adversarial setting but the lower bound
is stochastic. So it is a  reasonable lower bound to compare with.}.
Since $L=d$ when $\Phi=I$, the above results indicate that for general $\Phi$, the best upper bound one can hope is $O(\sqrt{Kdn})$. Hence, our bound is at most
$\tilde{O} ( \sqrt{K \min \{\ln(L), d \}})$ larger. 
It is well-known that the $O (\sqrt{d})$ factor is due to linear generalization \citep{dani08stochastic, abbasi-yadkori11improved},
and
as is discussed in the appendix (see Remark \ref{remark:remark1}), the extra $O(\sqrt{K})$ factor is also due to linear generalization. 
They might be intrinsic, but we leave the final word and tightness analysis to future work.

\subsection{Regret Bound on $\comblinucb$}
\label{sec:comblinucb}

Under the assumptions that
(1) the support of $P$ is a subset of 
$[0,1]^L$,
(2) the stochastic item weights $\left \{ \bw(e) \right \}_{e \in E}$ are statistically independent under $P$, and
(3) the oracle $\oracle$ \emph{exactly} solves the offline optimization problem\footnote{
If $\oracle$ is an approximation algorithm, a variant of Theorem \ref{thm:bound_comblinucb}
can be proved (see Appendix \ref{appendix:variant}).},
we have the following upper bound on $R(n)$ when $\comblinucb$ is applied to coherent learning cases:


\begin{theorem}
\label{thm:bound_comblinucb}
For any $\lambda, \sigma>0$, any $\delta \in (0,1)$, and any $c$ satisfying
\begin{align}
c \geq \frac{1}{\sigma} \sqrt{d \ln \left( 1+ \frac{n K \lambda^2}{d \sigma^2}\right) + 2 \ln \left( \frac{1}{\delta} \right)} + \frac{\left \| \theta^* \right \|_2}{\lambda} ,
\label{eq:c_lb}
\end{align}
if $\bar{\bw} = \Phi \theta^\ast$ and the above three assumptions hold, then 
under $\comblinucb$ algorithm with parameter $(\Phi, \lambda, \sigma, c)$, we have
\[
R(n) \leq 2c  K \lambda\sqrt{\frac{d n  \ln \left( 1 + \frac{nK \lambda^2}{d \sigma^2} \right)}{\ln \left(1+ \frac{\lambda^2}{\sigma^2} \right)}} + 
 nK \delta.
\]
\end{theorem}

Generally speaking, the proof for Theorem \ref{thm:bound_comblinucb} proceeds as follows. We first construct a confidence set $G$ of $\theta^\ast$ based on the ``self normalized bound" developed in \citet{abbasi-yadkori11improved}. Then we decompose the regret over the high-probability ``good" event $G$ and the low-probability
``bad" event $\bar{G}$, where $\bar{G}$ is the complement of $G$. Finally, we bound the term associated with the event $G$ based on the 
same worst-case bound on $\sum_{t=1}^n \sum_{e \in A^t} 
\sqrt{\phi_e^T \Sigma_t \phi_e}$ used in the analysis for $\comblints$ (see Lemma \ref{lem:key_ineq_1} in Appendix \ref{proof:combts}), and bound the term associated with the event $\bar{G}$ based on a naive bound.
Please refer to Appendix \ref{proof:comblinucb} for the detailed proof of Theorem \ref{thm:bound_comblinucb}.

Notice that if we choose $\lambda=\sigma=1$, $\delta=1/(nK)$, and $c$ as the lower bound specified in Inequality (\ref{eq:c_lb}), then the regret bound derived in Theorem \ref{thm:bound_comblinucb} is also 
$\tilde{O} ( K d \sqrt{n} )$. Compared with the lower bound derived in \citet{audibert14regret}, this bound is at most
$\tilde{O} ( \sqrt{Kd} )$ larger. Similarly, the extra $O ( \sqrt{K})$ and $O ( \sqrt{d})$ factors are also due to linear generalization.

Finally, we would like to clarify that the assumption that the support of $P$ is bounded is not essential.
By slightly modifying the analysis, we can achieve a similar high-probability bound on the \emph{realized cumulative regret} as long as
$P$ is \emph{sub-Gaussian}. We also want to point out that the $L$-independent bounds derived in both Theorem \ref{thm:bound_combts} and
\ref{thm:bound_comblinucb} will still hold even if $L=\infty$.


\section{Experiments}
\label{sec:experiments}

In this section, we evaluate $\comblints$ on three problems.
The first problem is synthetic, but the last two problems are constructed based on real-world datasets.
%
%
As we have discussed in Section \ref{sec:introduction}, we only evaluate $\comblints$
since in practice Thompson sampling algorithms usually outperform the UCB-like algorithms.
Our experiment results in the synthetic problem
demonstrate that $\comblints$ is both scalable and
robust to the choice of algorithm parameters. 
They also suggest the Bayes regret bound derived in Theorem \ref{thm:bound_combts} is likely to be tight.
On the other hand, our experiment results in the last two problems
show the  \emph{value of linear generalization} in real-world settings: with domain-specific but imperfect linear generalization (i.e. agnostic learning),
$\comblints$ can significantly outperform state-of-the-art learning algorithms that do not exploit linear generalization, which serve as baselines in these two problems.

In all three problems, the oracle $\oracle$ exactly solves the offline combinatorial optimization problem.
Moreover,
in the two real-world problems, we demonstrate the experiment results using a new performance metric,
the \emph{expected per-step return} in $n$ episodes, which is defined as
\begin{align}
  \frac{1}{n} \textstyle \E{\sum_{t = 1}^n f(A^t, \bw_t) \middle | \bar{\bw}  }{\bw_1, \dots, \bw_n }.
  \label{eq:per-step regret}
\end{align}
Obviously, it is the expected cumulative return in $n$ episodes divided by $n$.
We demonstrate experiment results using expected cumulative return rather than 
$R(n)$ 
 since it is more illustrative.

\subsection{Longest Path}
\label{sec:synthetic_experiment}
We first evaluate $\comblints$ on a synthetic problem. Specifically, 
we experiment with a stochastic longest path problem on an $(m+1) \times (m+1)$ square grid\footnote{That is, each side has $m$ edges and $m+1$ nodes. Notice that the longest path problem and the shortest path problem are mathematically equivalent.}.
The items in the ground set $E$ are the edges in the grid, $L=2m(m+1)$ in total. The feasible set $\cA$ are all paths 
in the grid from the upper left corner to the bottom right corner that follow the directions of the edges. The length of these paths is
$K=2m$. 
In this problem, we focus on coherent Gaussian cases and randomly sample the linear generalization matrix $\Phi \in \realset^{L \times d}$ 
to weaken the dependence on a particular choice of $\Phi$.

Our experiments are parameterized by
a sextuple $(m,d,\lambda_{\mathrm{true}}, \sigma_{\mathrm{true}}, \lambda, \sigma)$, where $m$, $d$, $\lambda$, and $\sigma$ are defined before
and $\lambda_{\mathrm{true}}$ and $\sigma_{\mathrm{true}}$ are respectively the true standard deviations of $\theta^\ast$ and the observation noises. In each round of simulation, we first construct a problem instance as follows: (1) generate $\Phi$ by sampling each component of $\Phi$ i.i.d. from $N(0,1)$; (2) sample
$\theta^\ast$ independently from $N(0, \lambda_{\mathrm{true}}^2 I)$ and set $\bar{\bw} = \Phi \theta^\ast$; and (3) $\forall (t,e)$, the observation noise $\eta_t(e)=\bw_t(e) -\bar{\bw}(e)$ is i.i.d. sampled from $N(0, \sigma_{\mathrm{true}}^2 )$. Then we apply $\comblints$ with parameter
$(\lambda, \sigma)$  to the constructed instance for $n$ episodes. 
Notice that in general $(\lambda, \sigma) \neq \left( \lambda_{\mathrm{true}}, \sigma_{\mathrm{true}} \right)$.
We average the experiment results over $200$ simulations to estimate the
Bayes cumulative regret $R_{\mathrm{Bayes}} (n)$.

We start with a ``default case" with $m=30$, $d=200$, $\lambda_{\mathrm{true}}=\lambda=10$ and 
$\sigma_{\mathrm{true}}=\sigma=1$. Notice in this case $L=1860$ and $|\cA| \approx 1.18 \times 10^{17}$.
We choose $n=150$ since in the default case, the Bayes per-episode regret of $\comblints$ vanishes far before period $150$.
In the default case $R_{\mathrm{Bayes}}(150) \approx 1.56 \times 10^4$.
In the experiments, 
 we vary only one and only one parameter while keeping all the other parameters fixed to their
``default values" specified above to
demonstrate the \emph{scalability}
and \emph{robustness} of $\comblints$.

\begin{figure}[t]
\centering
\subfigure[{\small $R_{\mathrm{Bayes}}$ vs. $m$}]
{
\label{fig:Fig_m}
\includegraphics[scale=0.307]{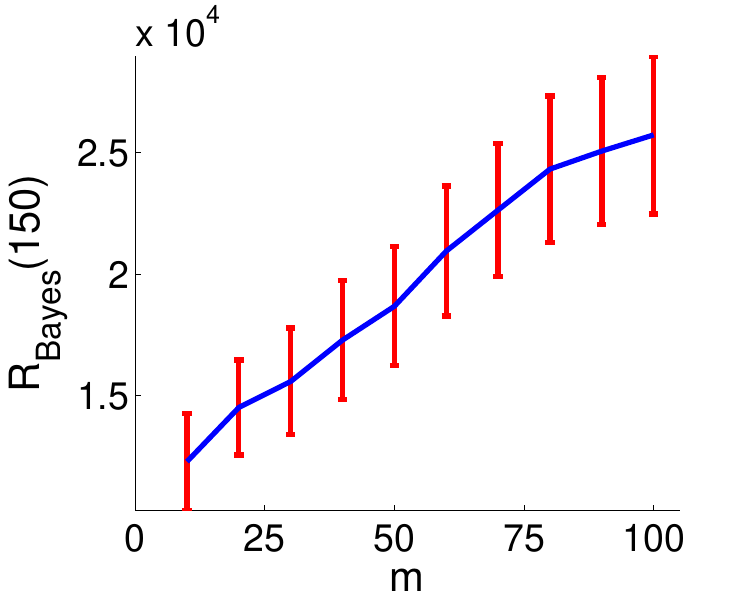}
}
\subfigure[{\small $R_{\mathrm{Bayes}}$ vs. $d$}]
{
\label{fig:Fig_d}
\includegraphics[scale=0.307]{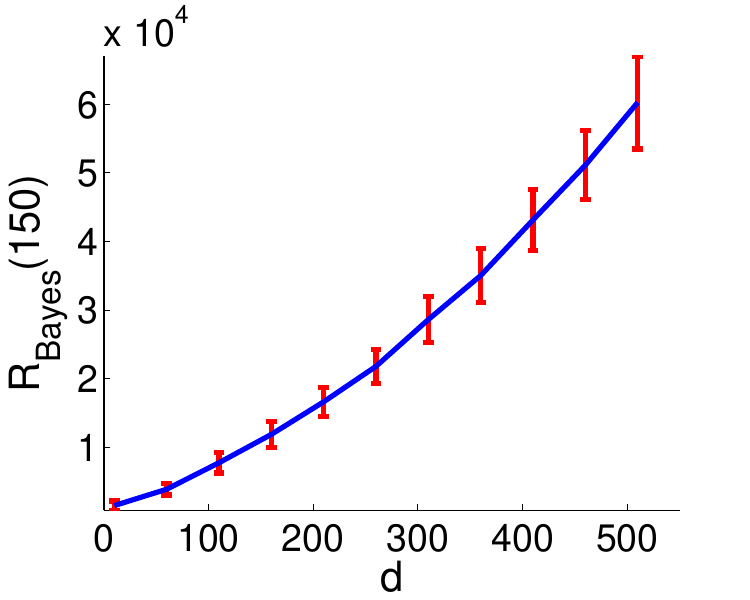}
}
\subfigure[{\small $R_{\mathrm{Bayes}}$ vs. $\sigma$}]
{
\label{fig:Fig_sigma}
\includegraphics[scale=0.307]{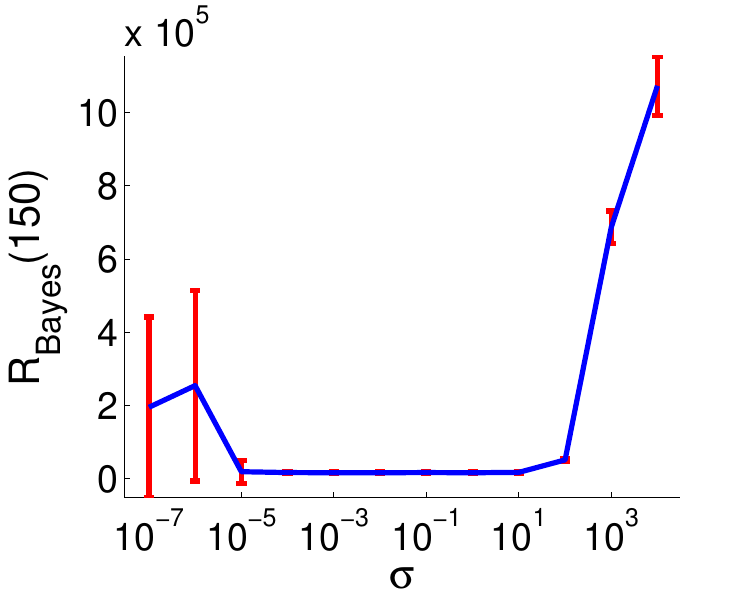}
}
\subfigure[{\small $R_{\mathrm{Bayes}}$ vs. $\lambda$}]
{
\label{fig:Fig_lambda}
\includegraphics[scale=0.307]{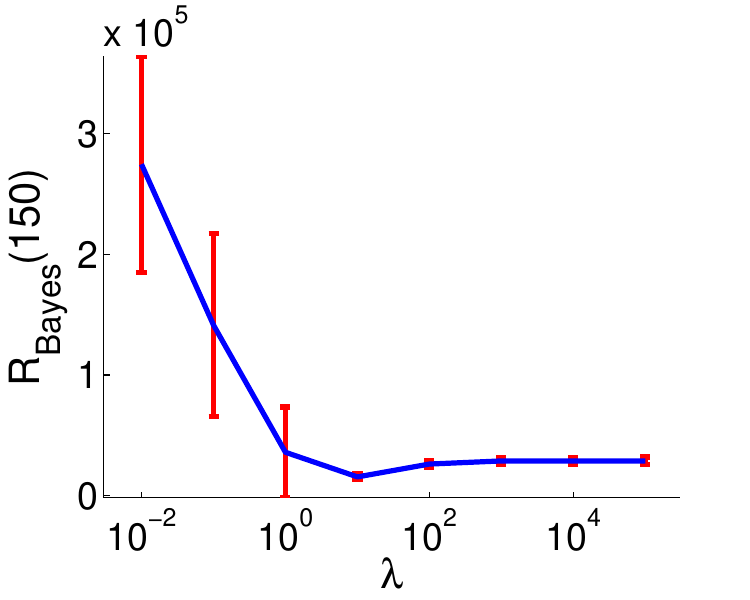}
}

\caption{Experiment results in the longest path problem}
\end{figure}

First, we study how the Bayes cumulative regret of $\comblints$ scales with the size of the problem
by varying $m=10,20,\ldots, 100$, and show the result in Figure \ref{fig:Fig_m}.
The experiment results show that  $R_{\mathrm{Bayes}}(150)$ roughly increases linearly with
$m$, which indicates that $\comblints$ is scalable with respect to the problem size $m$.
We also experiment with $m=250$, in this case we have $L \approx 125k$, $|\cA| \approx1.17 \times 10^{149}$,
and $R_{\mathrm{Bayes}}(150) \approx 6.56 \times 10^4$, which  is only $4.2$ times of $R_{\mathrm{Bayes}}(150) $ in the default case.
It is worth mentioning that this result also suggests that the Bayes regret bound derived in Theorem \ref{thm:bound_combts} is (almost) tight
in this problem\footnote{Recall that Theorem \ref{thm:bound_combts} requires $ \max_{e \in E} \| \phi_e\|_2 \leq 1$.
 It can be easily extended to cases with 
 $\max_{e \in E} \| \phi_e \|_2 \leq M$ by scaling the Bayes regret bound by $M$.
 However, in this problem $\phi_e$ is not bounded since it is sampled from a Gaussian distribution.
 We believe that Theorem \ref{thm:bound_combts} can be extended to this case by exploiting the properties of Gaussian distribution.
 Roughly speaking, in this problem, with high probability, $\|\phi_e \|_2 = O(\sqrt{d})$.
\label{footnote_1} }.
To see it, notice that $K=2m$ and $L=O (m^2)$, and hence the Bayes regret bound derived in Theorem \ref{thm:bound_combts}
is $\tilde{O} (m)$.

Second, we study how the Bayes cumulative regret of $\comblints$ scales with $d$, the dimension of the feature vectors, by
varying $d=10,60,110, \ldots, 510$, and demonstrate the result in Figure \ref{fig:Fig_d}.
The experiment results indicate that  $R_{\mathrm{Bayes}}(150)$ also roughly increases linearly with 
$d$, and hence $\comblints$ is also scalable with the feature dimension $d$.
This result also suggests that the $\tilde{O} (\sqrt{d})$ bound in Theorem \ref{thm:bound_combts} is (almost) tight\footref{footnote_1}.

%

Finally, we study the robustness of $\comblints$ with respect to the algorithm parameters $\sigma$ and $\lambda$.
In Figure \ref{fig:Fig_sigma}, we vary $\sigma=10^{-7}, 10^{-6}, \ldots, 10^4$ and in Figure \ref{fig:Fig_lambda}, we vary
$\lambda=10^{-2}, 10^{-1}, \ldots, 10^5$. We would like to emphasize again that we only vary the algorithm parameters and
fix $\sigma_{\mathrm{true}}=1$ and $\lambda_{\mathrm{true}}=10$.
The experiment results show that $\comblints$ is robust to the choice of algorithm parameters and performs well
for a wide range of $\sigma$ and $\lambda$. 
However, too small or too large $\sigma$, or too small $\lambda$, can significantly reduce the performance of $\comblints$, as 
we have discussed in Section \ref{sec:comblints}.

\subsection{Online Advertising}
\label{sec:ad_experiment}

In the second experiment, we evaluate $\comblints$ on an advertising problem. Our objective is to identify $100$ people that are most likely to accept an advertisement offer, subject to the targeting constraint that exactly half of them are females.
%
%
%
%
Specifically,
the ground set $E$ includes $33\text{k}$ representative people from Adult dataset \cite{ucimlrepository}, which was 
collected in the $1994$ US census. A feasible solution $A$ is any subset of $E$ with $\abs{A} = 100$ and satisfying the
targeting constraint mentioned above.
%
%
We assume that person $e$ accepts an advertisement offer with probability
\begin{align*}
  \bar{\bw}(e) =
  \begin{cases}
    0.15 & \text{income is at least $50$k} \\
    0.05 & \text{otherwise},
  \end{cases}
\end{align*}
and people accept offers independently of each other. The features in the generalization matrix $\Phi$ are the age, which is binned into $7$ groups; gender; whether the person works more than $40$ hours per week; and the length of education in years. All these features can be constructed based on the Adult dataset.

%

$\comblints$ is compared to three baselines. The first baseline is the optimal solution $A^{\mathrm{opt}}$. The second baseline is ${\tt CombUCB1}$ \cite{kveton15combinatorial}. This algorithm estimates the probability that person $e$ accepts the offer $\bar{\bw}(e)$ independently of the other probabilities. The third baseline is $\comblints$ without linear generalization, which we simply refer to as ${\tt CombTS}$. As in ${\tt CombUCB1}$, this algorithm estimates the probability that person $e$ accepts the offer $\bar{\bw}(e)$ independently of the other probabilities. The posterior of $\bar{\bw}(e)$ is modeled as a beta distribution.

\begin{figure}[t]
  \centering
  \includegraphics[width=3.2in, bb=2.25in 4.25in 6.25in 6.75in]{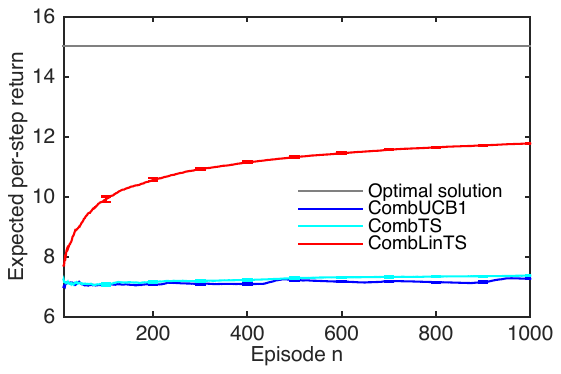}
  \caption{Regret on the ad dataset.}
  \label{fig:ad trends}
\end{figure}

Our experiment results are reported in Figure~\ref{fig:ad trends}. We observe two major trends. First, $\comblints$ learns extremely quickly.
In particular, its per-step return at episode $100$ is $70\%$ of the optimum, and
its per-step return at episode $1$k is $80\%$ of the optimum.
%
 %
 These results are remarkable since the linear generalization is imperfect in this problem.
 Second, both ${\tt CombUCB1}$ and ${\tt CombTS}$ perform poorly due to insufficient observations with respect to the model complexity.
 Specifically, in $1$k episodes, the people in $E$ are observed $100$k times, which implies that each person is observed only $3$ times on average. This is not enough to discriminate the people who are likely to accept the advertisement offer from those that are not.
 

\subsection{Artist Recommendation}
\label{sec:recommendations_experiment}

In the last experiment, we evaluate $\combts$ on a problem of recommending $K=10$ music artists that are most likely to be chosen by an average user of a music recommendation website.
%
%
Specifically,
the ground set $E$ include artists from the \emph{Last.fm} music recommendation dataset~\cite{Cantador:RecSys2011}. The dataset contains tagging and music artist listening information from a set of $2\text{k}$ users from Last.fm online music system\footnote{http://www.lastfm.com}. The tagging part includes the tag assignments of all artists provided by the users. For each user, the artists to whom she listened and the number of listening events are also available in the dataset. 

\begin{figure}[t]
  \centering
  \includegraphics[width=3.2in, bb=2.25in 4.25in 6.25in 6.75in]{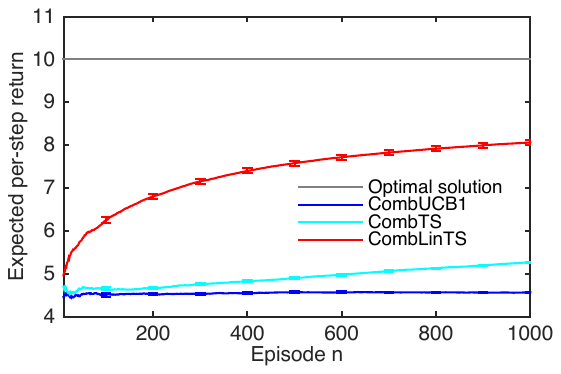}
  \caption{Regret on the recommendation experiment.}
  \label{fig:lastfm trends}
\end{figure}

We choose $E$ as the set of
artists that were listened by at least two users and had at least one tag assignment among the top $20$ most popular tags, and
$\abs{E} \approx 6$k. 
For each artist $e$, we construct its feature vector $\phi_e \in [0, 1]^{20}$ by setting its $j$th component as the fraction of users who assigned tag $j$ to this artist.  We assume that each artist $e$ is chosen by an average user with probability $\bar{\bw}(e) = \frac{1}{|U_e|} \sum_{u \in U_e} \bw_u(e)$, where $U_e$
is the set of users that listened to artist $e$, and $\bar{\bw}_u(e)$ is the probability that user $u$ likes artist $e$.
We estimate $\bar{\bw}_u(e)$ based on a  Na\"{\i}ve Bayes classifier with respect to the number of person/artist listening events.

%
%


Like Section \ref{sec:ad_experiment}, we also compare $\comblints$ to three baselines: the optimal solution $A^{\mathrm{opt}}$, the ${\tt CombUCB1}$ algorithm and the ${\tt CombTS}$ algorithm. Our experiment results are reported in Figure~\ref{fig:lastfm trends}. Similarly as Figure~\ref{fig:ad trends}, 
the expected per-step return of $\comblints$ approaches that of $A^{\mathrm{opt}}$ much faster than ${\tt CombUCB1}$ and ${\tt CombTS}$. Moreover,
both ${\tt CombUCB1}$ and ${\tt CombTS}$ perform poorly due to the insufficient observations with respect to the model complexity: In $1$k episodes, each 
artist is observed less than $2$ times on average, which is not enough to discriminate most popular artists from less popular artists.

%


\section{Conclusion}
\label{sec:conclusion}

We have proposed two learning algorithms, $\comblints$ and $\comblinucb$, for stochastic combinatorial semi-bandits with linear generalization.
The main contribution of this work is two-fold: First, we 
have established $L$-independent regret bounds for these two algorithms under reasonable assumptions, where $L$ is the number of items.
Second,
we have also evaluated $\comblints$ on a variety of problems.
The experiment results in the first problem show
that $\comblints$ is scalable and robust, and the experiment results in the other two problems
demonstrate the value of exploiting linear generalization in real-world settings.


It is worth mentioning that our results can be easily extended to the \emph{contextual combinatorial semi-bandits} with linear generalization.
In a contextual combinatorial semi-bandit, the probability distribution $P$ (and hence the expected weight $\bar{\bw}$) also depends
on a context $x$, which either follows an exogenous stochastic process or is adaptively chosen by an adversary.
Assume that each state-item pair $(x,e)$ is associated with a feature vector $\phi_{x,e}$, then similar to \citet{shipra2013}, 
both $\comblints$ and $\comblinucb$, as well as their analyses, can be generalized to the 
contextual combinatorial semi-bandits.

We leave open several questions of interest. 
One interesting open question is how to derive regret bounds for $\comblints$ and $\comblinucb$ in the agnostic learning cases.
Another interesting open question is how to extend the results to combinatorial semi-bandits with nonlinear generalization. 
We believe that our results can be extended to combinatorial semi-bandits with \emph{generalized linear generalization}\footnote{That is,
$\bar{\bw} (e) = h(\phi_e^T \theta^\ast)$, where $h: \, \realset \rightarrow \realset$ is a strictly monotone function.}, but leave it to future work.


\bibliography{Reference}
\bibliographystyle{icml2015}


\newpage
\onecolumn

\appendix

\newcommand{\trace}{\mathrm{trace}}

\section{Proof for Theorem \ref{thm:bound_combts}}
\label{proof:combts}
To prove Theorem \ref{thm:bound_combts}, we first prove the following theorem:

\begin{theorem}
\label{thm:bound_1_refine}
If (1) $\bar{\bw}=\Phi \theta^\ast$, (2) the prior on $\theta^\ast$ is 
$N(0, \lambda^2 I)$,
and (3) the noises are i.i.d. sampled from $N(0 , \sigma^2)$,
then under $\combts$ algorithm with parameter $(\Phi, \lambda, \sigma)$, then we have
\begin{align}
R_{\mathrm{Bayes}}(n) \leq 1+K \lambda
\min \left \{
\sqrt{\ln \left(\frac{\lambda L n}{\sqrt{2 \pi}} \right)},
\sqrt{d \ln \left( \frac{2d Kn \lambda}{\sqrt{2 \pi}}\right)}
\right \}
\sqrt{\frac{2d n  \ln \left( 1 + \frac{nK \lambda^2}{d} \right)}{\ln\left(1+ \frac{\lambda^2}{\sigma^2} \right)}}.
\label{eq:bound}
\end{align}
\end{theorem}
Notice that Theorem \ref{thm:bound_combts} follows immediately from Theorem \ref{thm:bound_1_refine}.
Specifically, if $\lambda \geq \sigma$, then we have
\begin{align}
B_{\mathrm{Bayes}}(n) & \leq 1+K \lambda
\min \left \{
\sqrt{\ln \left(\frac{\lambda L n}{\sqrt{2 \pi}} \right)},
\sqrt{d \ln \left( \frac{2d Kn \lambda}{\sqrt{2 \pi}}\right)}
\right \}
\sqrt{2d n \log_2 \left( 1 + \frac{nK \lambda^2}{d} \right)} \nonumber \\
&=
\tilde{O} \left(
K \lambda \sqrt{dn \min \left\{\ln(L), d \right\}}
\right) .
\end{align}

We now outline the proof of Theorem \ref{thm:bound_1_refine}, which is based on
\cite{russo2013posterior,dani08stochastic}.
Let $\cH_t$ denote the ``history" (i.e. all the available information) by the start of episode $t$. Note that from the Bayesian perspective, conditioning on $\cH_t$, $\theta^\ast$
and $\theta_t$ are i.i.d. drawn from $N(\bar{\theta}_t, \Sigma_t)$ (see \cite{russo2013posterior}).
This is because that conditioning on $\cH_t$, the posterior belief in $\theta^\ast$ is $N(\bar{\theta}_t, \Sigma_t)$ and based on Algorithm \ref{algorithm:combts}, $\theta_t$ is independently sampled from $N(\bar{\theta}_t, \Sigma_t)$. Since $\oracle$ is a fixed combinatorial optimization algorithm (even though it can be independently randomized), and $E,\cA, \Phi$ are all fixed, then 
conditioning on $\cH_t$,
$A^\ast$ and $A^t$ are also i.i.d., furthermore, $A^\ast$ is conditionally independent of $\theta_t$, and $A^t$ is conditionally independent of $\theta^\ast$.

To simplify the exposition, $\forall \theta \in \realset^d$ and $\forall A \subseteq E$, we define
\begin{align}
g(A , \theta)=\sum_{e \in A} \left< \phi_e, \theta \right>,
\end{align}
then we have
$\E{f(A^\ast, \bw_t) \middle | \cH_t, \theta^\ast, \theta_t, A^\ast, A^t}{}=g(A^\ast, \theta^\ast)$ and 
$\E{f(A^t, \bw_t) \middle | \cH_t, \theta^\ast, \theta_t, A^\ast, A^t}{}=g(A^t, \theta^\ast)$, hence we have
$\E{R_t \middle | \cH_t}{}=\E{g(A^\ast, \theta^\ast) - g(A^t, \theta^\ast) \middle | \cH_t}{}$. We also define the \emph{upper confidence bound (UCB)} function
$U_t : 2^E \rightarrow \realset$ as
\begin{align}
U_t (A)=\sum_{e \in A} \left[  \left<\phi_e, \bar{\theta}_t \right> + c 
\sqrt{\phi_e^T \Sigma_t \phi_e} \right],
\end{align}
where $c>0$ is a constant to be specified. Notice that conditioning on $\cH_t$, $U_t$ is a deterministic function and $A^\ast, A^t$ are i.i.d., then 
$\E{U_t (A^t)- U_t (A^\ast) \middle | \cH_t}{}=0$ and 
\begin{align}
\E{ R_t \middle | \cH_t}{} 
= \E{ g(A^\ast, \theta^\ast)- U_t (A^\ast)  \middle | \cH_t}{}
+
\E{U_t (A^t)-g(A^t, \theta^*)  \middle | \cH_t}{}.
\end{align}
One key observation is that
\begin{align}
\E {U_t (A^t)-g(A^t, \theta^*) \middle | \cH_t}{} &\stackrel{(a)}{=}
\sum_{e \in E} \E{ \I{e \in A^t} \left[ \left< \phi_e, \bar{\theta}_t - \theta^\ast \right> + c 
\sqrt{\phi_e^T \Sigma_t \phi_e} \right] \middle | \cH_t}{} \nonumber \\
&\stackrel{(b)}{=}
\sum_{e \in E} \E{ \I{e \in A^t} \middle | \cH_t}{} \E{ \left< \phi_e, \bar{\theta}_t - \theta^\ast \right> \middle | \cH_t}{} + c \E{ \sum_{e \in A^t} 
\sqrt{\phi_e^T \Sigma_t \phi_e} \middle | \cH_t}{} \nonumber \\
&\stackrel{(c)}{=}  c \E{ \sum_{e \in A^t} 
\sqrt{\phi_e^T \Sigma_t \phi_e} \middle | \cH_t}{},
\end{align}
where (b) follows from the fact that $A^t$ and $\theta^\ast$ are conditionally independent, and (c) follows from $\E{\theta^\ast \middle | \cH_t}{}=\bar{\theta}_t$.
Hence $B_{\mathrm{Bayes}}(n)=\sum_{t=1}^n \E{ g(A^\ast, \theta^\ast)- U_t (A^\ast)}{} + c \sum_{t=1}^n \E{\sum_{e \in A^t} 
\sqrt{\phi_e^T \Sigma_t \phi_e}}{}$. We can show that (1) $\sum_{t=1}^n \E{ g(A^\ast, \theta^\ast)- U_t (A^\ast)}{} \leq 1$ if we choose
\begin{align}
c \geq \min \left \{
\sqrt{\ln \left(\frac{\lambda L n}{\sqrt{2 \pi}} \right)},
\sqrt{d \ln \left( \frac{2d Kn \lambda}{\sqrt{2 \pi}}\right)}
\right \},
\end{align}
and (2) $\sum_{t=1}^n \E{\sum_{e \in A^t} 
\sqrt{\phi_e^T \Sigma_t \phi_e}}{} \leq
K \lambda \sqrt{2d n  \ln \left( 1 + \frac{nK \lambda^2}{d} \right) /\ln\left(1+ \frac{\lambda^2}{\sigma^2} \right)}
$.
Thus,
the bound in Theorem \ref{thm:bound_1_refine} holds. 
Please refer to the remainder of this section for the full proof.

\subsection{Bound on $\sum_{t=1}^n \E{ g(A^\ast, \theta^\ast)- U_t (A^\ast)}{}$}
We first prove that if we choose
\begin{align}
c \geq \min \left \{
\sqrt{\ln \left(\frac{\lambda L n}{\sqrt{2 \pi}} \right)},
\sqrt{d \ln \left( \frac{2d Kn \lambda}{\sqrt{2 \pi}}\right)}
\right \}, 
\end{align}
then $\sum_{t=1}^n \E{ g(A^\ast, \theta^\ast)- U_t (A^\ast)}{} \leq 1$.
To prove this result, we use the following inequality for truncated Gaussian distribution.
\begin{lemma}
\label{lemma:tech1}
If $X \sim N(\mu, s^2)$, then we have
\[ \E{X \I{X \geq 0}}{}=\mu \left[ 1- \Phi_G \left(\frac{-\mu}{s} \right) \right]+\frac{s}{\sqrt{2 \pi}} \exp \left( - \frac{\mu^2}{2 s^2}\right),\]
where $\Phi_G$ is the cumulative distribution function (CDF) of the standard Gaussian distribution $N(0,1)$. Furthermore, if $\mu \leq 0$, we have
$\E{X \I{X \geq 0}}{} \leq \frac{s}{\sqrt{2 \pi}} \exp \left( - \frac{\mu^2}{2 s^2}\right)$. 
\end{lemma}
Based on Lemma \ref{lemma:tech1}, we can prove the following lemmas:
\begin{lemma}
\label{lemma:tech2}
If $c \geq \sqrt{\ln \left(\frac{\lambda L n}{\sqrt{2 \pi}} \right)}$, then we have
$\sum_{t=1}^n \E{ g(A^\ast, \theta^\ast)- U_t (A^\ast)}{} \leq 1$.
\end{lemma}
\proof
We have the following naive bound:
\begin{align}
g(A^\ast, \theta^\ast)- U_t (A^\ast)&=\sum_{e \in A^\ast}
\left[ \left <
\phi_e, \theta^\ast - \bar{\theta}_t
\right > - c \sqrt{\phi_e^T \Sigma_t \phi_e}
\right] \nonumber \\
&\leq 
\sum_{e \in A^\ast}
\left[ \left <
\phi_e, \theta^\ast - \bar{\theta}_t
\right > - c\sqrt{\phi_e^T \Sigma_t \phi_e}
\right]
\I{
\left <
\phi_e, \theta^\ast - \bar{\theta}_t
\right > - c \sqrt{\phi_e^T \Sigma_t \phi_e} \geq 0
} \nonumber \\
& \leq
\sum_{e \in E}
\left[ \left <
\phi_e, \theta^\ast - \bar{\theta}_t
\right > - c \sqrt{\phi_e^T \Sigma_t \phi_e}
\right]
\I{
\left <
\phi_e, \theta^\ast - \bar{\theta}_t
\right > - c \sqrt{\phi_e^T \Sigma_t \phi_e} \geq 0
}. \nonumber
\end{align}
Notice that conditioning on $\cH_t$, 
$\left <
\phi_e, \theta^\ast - \bar{\theta}_t
\right > - c \sqrt{\phi_e^T \Sigma_t \phi_e}$ is a Gaussian random variable with mean $- c \sqrt{\phi_e^T \Sigma_t \phi_e}$ and variance
$\phi_e^T \Sigma_t \phi_e$.
Thus, from Lemma \ref{lemma:tech1}, we have
\begin{align}
& \E{g(A^\ast, \theta^\ast)- U_t (A^\ast) \middle | \cH_t}{\theta^\ast, A^\ast}  
\nonumber \\
\stackrel{(a)}{\leq} &
\sum_{e \in E} 
\E{ \left[ \left <
\phi_e, \theta^\ast - \bar{\theta}_t
\right > - c \sqrt{\phi_e^T \Sigma_t \phi_e}
\right]
\I{
\left <
\phi_e, \theta^\ast - \bar{\theta}_t
\right > - c \sqrt{\phi_e^T \Sigma_t \phi_e} \geq 0
}\middle | \cH_t}{\theta^\ast} \nonumber \\
\stackrel{(b)}{\leq} &
\sum_{e \in E} 
\sqrt{\frac{\phi_e^T \Sigma_t \phi_e}{2 \pi}} \exp \left(-\frac{c^2}{2} \right) \nonumber \\
\stackrel{(c)}{\leq} &  \exp \left(-\frac{c^2}{2} \right) \sum_{e \in E} \frac{\lambda \| \phi_e \|}{\sqrt{2 \pi}} \leq  \exp \left(-\frac{c^2}{2} \right) \frac{\lambda L}{\sqrt{2 \pi}}, \label{eq:partial_new_1}
\end{align}
where the last two inequalities follow from the fact that
$\phi_e^T \Sigma_t \phi_e \leq \phi_e^T \Sigma_1 \phi_e \leq \lambda^2 \| \phi_e \|^2 \leq \lambda^2$, since $\| \phi_e \| \leq 1$ by assumption\footnote{Notice that in the derivation of  Inequality (\ref{eq:partial_new_1}), we implicitly assume that $\phi_e^T \Sigma_t \phi_e>0$, $\forall e \in E$. It is worth pointing out that the case with $\phi_e^T \Sigma_t \phi_e=0$ is a trivial case and this inequality still holds in this case.}.
Thus we have
\begin{align}
\E{ \sum_{t=1}^n \left[ g(A^\ast, \theta^\ast)- U_t (A^\ast) \right]}{} \leq &
\exp \left(-\frac{c^2}{2} \right) \frac{n \lambda L}{\sqrt{2 \pi}}.
\end{align}
If we choose
$
c \geq \sqrt{2 \ln \left(\frac{\lambda L n}{\sqrt{2 \pi}} \right)}
$,
then we have
$
\E{ \sum_{t=1}^n \left[ g(A^\ast, \theta^\ast)- U_t (A^\ast) \right]}{} \leq 1$.
\endproof

\begin{lemma}
\label{lemma:tech3}
If $c \geq \sqrt{d \ln \left( \frac{2d Kn \lambda}{\sqrt{2 \pi}}\right)}$, then we also have
$\sum_{t=1}^n \E{ g(A^\ast, \theta^\ast)- U_t (A^\ast)}{} \leq 1$.
\end{lemma}
\proof
We use $v_1, \ldots, v_d$ to denote a fixed set of $d$ orthonormal
eigenvectors of $\Sigma_t$, and $\Lambda_1^2, \ldots, \Lambda_d^2$ to denote the associated eigenvalues. Notice that for $i \neq j$, we have
$v_i^T \Sigma_t v_j=\Lambda_i^2 v_i^T  v_j=0$.
$\forall i=1,\ldots, d$, we define $v_{i+d}=-v_i$ and $\Lambda_{i+d}=\Lambda_i$, 
which allows us to define the following conic decomposition:
\[
\phi_e=\sum_{i=1}^{2d} \alpha_{ei} v_i, \quad \forall e \in E,
\]
subject to the constraints that $\alpha_{ei} \geq 0$, $\forall (e,i)$.
Notice that $\alpha_{ei}$'s are uniquely determined.
Furthermore, for $i$ and $j$ s.t. $|i-j|=d$, by definition of conic decomposition,
we have $\alpha_{ei} \alpha_{ej}=0$. In other words, $\alpha_e$ is a $d$-sparse vector.

Since we assume that $\|\phi_e \| \leq 1$, we have that
$\sum_{i=1}^{2d} \alpha_{ei}^2 \leq 1$, $\forall e \in E$. Thus, for any $e$, we have that
$\left< \phi_e, \theta^\ast- \bar{\theta}_t\right>=\sum_{i=1}^{2d} \alpha_{ei} \left< v_i, \theta^\ast -\bar{\theta}_t \right>$ and
\begin{align}
\phi_e^T \Sigma_t \phi_e &=
\left( \sum_{i=1}^{2d} \alpha_{ei} v_i^T \right) \Sigma_t
\left( \sum_{j=1}^{2d} \alpha_{ei} v_j \right) \nonumber \\
&=\sum_{i=1}^{2d} \sum_{j=1}^{2d} \alpha_{ei} \alpha_{e_j} v_i^T \Sigma_t v_j .
\end{align}
Notice that for $i \neq j$, if $|i-j| \neq d$, then 
$v_i^T \Sigma_t v_j =0$; on the other hand, if $|i-j| = d$,
$\alpha_{ei} \alpha_{ej}=0$. Thus, if $i \neq j$, we have 
$\alpha_{ei} \alpha_{e_j} v_i^T \Sigma_t v_j=0$.
Consequently,
$$ \phi_e^T \Sigma_t \phi_e=\sum_{i=1}^{2d}  \alpha_{ei}^2 v_i^T \Sigma_t v_i =\sum_{i=1}^{2d}  \alpha_{ei}^2  \Lambda^2_i.$$

Thus we have
\begin{align}
\sqrt{\phi_e^T \Sigma_t \phi_e} =
\sqrt{\sum_{i=1}^{2d} \alpha^2_{ei} \Lambda_i^2} \geq \frac{1}{\sqrt{d}} \sum_{i=1}^{2d} \alpha_{ei}  \Lambda_i,
\end{align}
where the inequality follows from Cauchy-Schwartz inequality, specifically, define $s_i=1$ if $\alpha_{ei} \Lambda_i \neq 0$, and $s_i=0$ if
$\alpha_{ei} \Lambda_i = 0$, then we have
\[
\sum_{i=1}^{2d} \alpha_{ei} \Lambda_{i} = \sum_{i=1}^{2d} \alpha_{ei} \Lambda_i s_i
\leq \sqrt {\sum_{i=1}^{2d} s_i^2} \sqrt {\sum_{i=1}^{2d}\alpha_{ei}^2 \Lambda_i^2} \leq
 \sqrt {d} \sqrt {\sum_{i=1}^{2d}\alpha_{ei}^2 \Lambda_i^2},
\]
where the last inequality follows from the fact that $\alpha_e$ is $d$-sparse.
Thus, for any $e$, we have that
\begin{align}
\left <
\phi_e, \theta^\ast - \bar{\theta}_t
\right > - c \sqrt{\phi_e^T \Sigma_t \phi_e}
\leq \sum_{i=1}^{2d} \alpha_{ei} \left< v_i, \theta^\ast -\bar{\theta}_t \right> -
\frac{c}{\sqrt{d}} \sum_{i=1}^{2d} \alpha_{ei} \Lambda_i.
\end{align}
Consequently, we have
\begin{align}
\sum_{e \in A^\ast} \left[  \left <
\phi_e, \theta^\ast - \bar{\theta}_t
\right > - c \sqrt{\phi_e^T \Sigma_t \phi_e} \right]
\leq
\sum_{i=1}^{2d}  \left( \left< v_i, \theta^\ast -\bar{\theta}_t \right> -\frac{c \Lambda_i}{\sqrt{d}} \right)
\left( \sum_{e \in A^\ast} \alpha_{ei} \right).
\end{align}
Define $X_i=\left< v_i, \theta^\ast -\bar{\theta}_t \right> -\frac{c \Lambda_i}{\sqrt{d}} $, 
notice that
conditioning on $\cH_t$, we have
$X_i | \cH_t \sim N \left(-\frac{c \Lambda_i}{\sqrt{d}}, \Lambda_i^2 \right)$.
Hence we have
\begin{align}
\sum_{e \in A^\ast} \left <
\phi_e, \theta^\ast - \bar{\theta}_t
\right > - c \sqrt{\phi_e^T \Sigma_t \phi_e} & \stackrel{(a)}{\leq}
\sum_{i=1}^{2d} X_i \left[ \sum_{e \in A^\ast} \alpha_{ei} \right] \nonumber \\
& \stackrel{(b)}{\leq} 
\sum_{i=1}^{2d} X_i 
\I{ X_i  \geq 0}
\left[ \sum_{e \in A^\ast} \alpha_{ei} \right], \nonumber
\end{align}
where the inequality (b) follows from the fact that
$X_i  \leq X_i \I{X_i \geq 0}$
and $\left[ \sum_{e \in A^\ast} \alpha_{ei} \right] \geq 0$.
On the other hand, notice that $|A^\ast| \leq K$
\[
 \sum_{e \in A^\ast} \alpha_{ei}   \leq \sqrt{|A^\ast|} \sqrt{\sum_{e \in A^\ast} \alpha_{ei}^2}
\leq 
\sqrt{|A^\ast|} \sqrt{\sum_{e \in A^\ast} \sum_{j=1}^d \alpha_{ej}^2} \leq 
\sqrt{|A^\ast|} \sqrt{\sum_{e \in A^\ast} 1}=|A^\ast| \leq K.
\]
Since $X_i  
\I{ X_i \geq 0} \geq 0$, we have
\begin{align}
\sum_{e \in A^\ast} \left <
\phi_e, \theta^\ast - \bar{\theta}_t
\right > - c \sqrt{\phi_e^T \Sigma_t \phi_e} 
\leq 
K \sum_{i=1}^{2d}  X_i  
\I{ X_i \geq 0} 
, \nonumber
\end{align}
notice that the RHS does not include $A^\ast$. Hence we have
\begin{align}
\E{g(A^\ast, \theta^\ast)- U_t (A^\ast) \middle | \cH_t}{\theta^\ast} &= 
\E{\sum_{e \in A^\ast} \left <
\phi_e, \theta^\ast - \bar{\theta}_t
\right > - c \sqrt{\phi_e^T \Sigma_t \phi_e} \middle | \cH_t}{\theta^\ast} \nonumber \\
& \leq K \sum_{i=1}^{2d} \E{X_i \I{X_i \geq 0} \middle | \cH_t}{\theta^\ast} \nonumber \\
& \leq K \sum_{i=1}^{2d} \frac{\Lambda_i}{\sqrt{2 \pi}} \exp \left ( -\frac{c^2}{2d} \right ) \nonumber  \leq \frac{2d K \lambda}{\sqrt{2 \pi}} \exp \left ( -\frac{c^2}{2d} \right ),
\end{align} 
where the last inequality follows from the fact that $\Lambda_i \leq \lambda$. Hence we have
\[
\sum_{t=1}^n \E{g(A^\ast, \theta^\ast)- U_t (A^\ast)}{} \leq 
\frac{2d K n \lambda}{\sqrt{2 \pi}} \exp \left ( -\frac{c^2}{2d} \right ),
\]
if we choose
$
c \geq \sqrt{2d \ln \left( \frac{2d Kn \lambda}{\sqrt{2 \pi}}\right)}$, 
then we have
$
\sum_{t=1}^n \E{f(A^\ast, \theta^\ast)- U_t (A^\ast)}{} \leq 
1$.
\endproof

Combining the results from Lemma \ref{lemma:tech2} and \ref{lemma:tech3}, we have proved that if 
\[c \geq \min \left \{
\sqrt{\ln \left(\frac{\lambda L n}{\sqrt{2 \pi}} \right)},
\sqrt{d \ln \left( \frac{2d Kn \lambda}{\sqrt{2 \pi}}\right)}
\right \}, \]
then $\sum_{t=1}^n \E{ g(A^\ast, \theta^\ast)- U_t (A^\ast)}{} \leq 1$.


\subsection{Bound on $\sum_{t=1}^n \E{\sum_{e \in A^t} 
\sqrt{\phi_e^T \Sigma_t \phi_e}}{}$}
In this subsection, we derive a bound on $\sum_{t=1}^n \E{\sum_{e \in A^t} 
\sqrt{\phi_e^T \Sigma_t \phi_e}}{}$. Our analysis is motivated by the analysis in \cite{dani08stochastic}.
Specifically, we provide a worst-case bound on $\sum_{t=1}^n \sum_{e \in A^t} 
\sqrt{\phi_e^T \Sigma_t \phi_e}$, for any realization of random variable
$\bw_t$'s, $\theta_t$'s, $A^t$'s, $A^\ast$, and $\theta^\ast$.

\begin{lemma}
\label{lem:key_ineq_1}
$\sum_{t=1}^n \sum_{e \in A^t} \sqrt{\phi_e^T \Sigma_t \phi_e} \leq K \lambda\sqrt{\frac{d n  \log \left( 1 + \frac{nK \lambda^2}{d \sigma^2} \right)}{\log\left(1+ \frac{\lambda^2}{\sigma^2} \right)}}$.
\end{lemma}
\begin{proof}
To simplify the exposition, we define
\begin{align}
z_{t,k}=\sqrt{\phi_{a^t_k}^T \Sigma_t \phi_{a^t_k}}.
\end{align}
First, notice that $\Sigma_t^{-1}$ is the Gramian matrix and satisfies
\begin{align}
\Sigma_{t+1}^{-1}=\Sigma_t^{-1}+\frac{1}{\sigma^2} \sum_{k=1}^{|A^t|}
\phi_{a^t_k}\phi_{a^t_k}^T.
\end{align}
Hence for any $t$, $k$, we have that
\begin{align}
\det \left[ \Sigma_{t+1}^{-1} \right] &\geq
\det \left[
\Sigma_t^{-1}+\frac{1}{\sigma^2} 
\phi_{a^t_k}\phi_{a^t_k}^T
\right] 
=
\det \left[ \Sigma_t^{-\frac{1}{2}}
\left(
I+\frac{1}{\sigma^2} \Sigma_t^{\frac{1}{2}}
\phi_{a^t_k}\phi_{a^t_k}^T
\Sigma_t^{\frac{1}{2}}
\right)
\Sigma_t^{-\frac{1}{2}}
\right] \nonumber \\
&=
\det \left[ \Sigma_t^{-1}\right]
\det \left[
I+\frac{1}{\sigma^2} \Sigma_t^{\frac{1}{2}}
\phi_{a^t_k}\phi_{a^t_k}^T
\Sigma_t^{\frac{1}{2}}
\right] 
=
\det \left[ \Sigma_t^{-1}\right]
\left(
1+ \frac{1}{\sigma^2} \phi_{a^t_k}^T
\Sigma_t \phi_{a^t_k}
\right) \nonumber \\
&=
\det \left[ \Sigma_t^{-1}\right]
\left(
1+ \frac{z_{t,k}^2}{\sigma^2} 
\right).
\end{align}
Hence we have that
\begin{align}
\left( \det \left[ \Sigma_{t+1}^{-1} \right] \right)^{|A^t|} \geq
\left( \det \left[ \Sigma_{t}^{-1} \right] \right)^{|A^t|} \prod_{k=1}^{|A^t|}
\left(
1+ \frac{z_{t,k}^2}{\sigma^2} 
\right). \label{eq:temporary}
\end{align}
\normalsize
\begin{remark}
\label{remark:remark1}
This is where the extra $O(\sqrt{K})$ factor arises. Notice that this extra factor is purely due to linear generalization. Specifically, if $\Phi=I$, then $\Sigma_t$'s and $\Sigma_t^{-1}$'s will be diagonal, and we have
\small
\begin{align}
 \det \left[ \Sigma_{t+1}^{-1} \right]  =
 \det \left[ \Sigma_{t}^{-1} \right]  \prod_{k=1}^{|A^t|}
\left(
1+ \frac{z_{t,k}^2}{\sigma^2} 
\right). 
\end{align}
\end{remark}

Notice that Equation \ref{eq:temporary}
further implies that
\begin{align}
\left( \det \left[ \Sigma_{t+1}^{-1} \right] \right)^{K} \geq
\left( \det \left[ \Sigma_{t}^{-1} \right] \right)^{K} \prod_{k=1}^{|A^t|}
\left(
1+ \frac{z_{t,k}^2}{\sigma^2} 
\right),
\end{align}
since $\det \left[ \Sigma_{t+1}^{-1} \right] \geq \det \left[ \Sigma_{t}^{-1} \right]$ and $|A^t|\leq K$. Recall that $\det \left[\Sigma_1^{-1} \right]=\left( \frac{1}{\lambda^2}\right)^d$, we have that
\begin{align}
\left( \det \left[ \Sigma_{n+1}^{-1} \right] \right)^{K} \geq
\left( \det \left[ \Sigma_{1}^{-1} \right] \right)^{K} \prod_{t=1}^n \prod_{k=1}^{|A^t|}
\left(
1+ \frac{z_{t,k}^2}{\sigma^2} 
\right)=
\frac{1}{\lambda^{2d K}} \prod_{t=1}^n \prod_{k=1}^{|A^t|}
\left(
1+ \frac{z_{t,k}^2}{\sigma^2} 
\right).
\end{align}

On the other hand, we have
\begin{align}
\trace \left[ \Sigma_{n+1}^{-1} \right] &=
\trace \left[ \frac{1}{\lambda^2} I + \frac{1}{\sigma^2} \sum_{t=1}^n \sum_{k=1}^{|A^t|} \phi_{a^t_k} \phi_{a^t_k}^T \right] 
= \frac{d}{\lambda^2} + \frac{1}{\sigma^2} \sum_{t=1}^n \sum_{k=1}^{|A^t|} \| \phi_{a^t_k} \|^2 
\leq \frac{d}{\lambda^2} + \frac{nK}{\sigma^2},
\end{align}
where the last inequality follows from the assumption that 
$\| \phi_e \| \leq 1$, $\forall e \in E$ and $|A^t|\leq K$. From the trace-determinant inequality, we have
\[
\frac{1}{d} \trace \left[ \Sigma_{n+1}^{-1} \right] \geq 
\left ( \det \left[ \Sigma_{n+1}^{-1} \right] \right )^{\frac{1}{d}},
\]
which implies that
\begin{align}
\left( \frac{1}{\lambda^2} + \frac{nK}{d \sigma^2} \right)^{dK} \geq
\left( \frac{1}{d} \trace \left[ \Sigma_{n+1}^{-1} \right] \right)^{dK}
\geq \left ( \det \left[ \Sigma_{n+1}^{-1} \right] \right )^{K} \geq
\frac{1}{\lambda^{2d K}} \prod_{t=1}^n \prod_{k=1}^{|A^t|} \left(
1+ \frac{z_{t,k}^2}{\sigma^2} 
\right). \nonumber
\end{align}
Taking the logarithm, we have
\begin{align}
dK \log \left( 1 + \frac{nK \lambda^2}{d \sigma^2} \right) \geq 
\sum_{t=1}^n \sum_{k=1}^{|A^t|} \log \left(
1+ \frac{z_{t,k}^2}{\sigma^2} 
\right).
\end{align}
Notice that $z_{t,k}^2 = \phi_{a^t_k}^T \Sigma_t \phi_{a^t_k}$, hence we have that
$0 \leq z_{t,k}^2 \leq \phi_{a^t_k}^T \Sigma_1 \phi_{a^t_k} \leq \lambda^2 \| \phi_{a^t_k} \|^2 \leq \lambda^2$.
We have the following technical lemma:
\begin{lemma}
For any real number $x \in [0 , \lambda^2]$, we have
$x \leq \frac{\lambda^2}{\log\left(1+ \frac{\lambda^2}{\sigma^2} \right)} \log \left(1+\frac{x}{\sigma^2} \right)$.
\end{lemma}
\proof
Define $h(x) = \frac{\lambda^2}{\log\left(1+ \frac{\lambda^2}{\sigma^2} \right)} \log \left(1+\frac{x}{\sigma^2} \right)-x$, thus we only need to prove $h(x) \geq 0$ for $x \in [0 , \lambda^2]$. Notice that $h(x)$ is a strictly concave function for $x \geq 0$, and $h(0)=0$, $h(\lambda^2)=0$. From Jensen's inequality, for any $x \in (0, \lambda^2)$, we have $h(x)>0$.
\endproof
Hence we have that
\begin{align}
\sum_{t=1}^n \sum_{k=1}^{|A^t|} z_{t,k}^2 & \leq \frac{\lambda^2}{\log\left(1+ \frac{\lambda^2}{\sigma^2} \right)} \sum_{t=1}^n \sum_{k=1}^{|A^t|} \log \left(
1+ \frac{z_{t,k}^2}{\sigma^2} 
\right) 
 \leq 
\frac{dK \lambda^2  \log \left( 1 + \frac{nK \lambda^2}{d \sigma^2} \right)}{\log\left(1+ \frac{\lambda^2}{\sigma^2} \right)}
\end{align}
Finally, we have that
\begin{align}
\sum_{t=1}^n \sum_{k=1}^{|A^t|} z_{t,k} &\leq \sqrt{nK} \sqrt{\sum_{t=1}^n \sum_{k=1}^{|A^t|} z_{t,k}^2} 
\leq K \lambda\sqrt{\frac{d n  \log \left( 1 + \frac{nK \lambda^2}{d \sigma^2} \right)}{\log\left(1+ \frac{\lambda^2}{\sigma^2} \right)}}.
\end{align}
\end{proof}

Recall that the above bound holds for any realization of random variables, thus, we have
\[
\E{\sum_{t=1}^n \left[ U_t (A^t)-g(A^t, \theta^*) \right]  }{}
= c \E{\sum_{t=1}^n \sum_{k=1}^{|A^t|} z_{t,k}}{} \leq c K \lambda\sqrt{\frac{d n  \log \left( 1 + \frac{nK \lambda^2}{d} \right)}{\log\left(1+ \frac{\lambda^2}{\sigma^2} \right)}}.
\]
With 
\begin{align}
c = \min \left \{
\sqrt{\ln \left(\frac{\lambda L n}{\sqrt{2 \pi}} \right)},
\sqrt{d \ln \left( \frac{2d Kn \lambda}{\sqrt{2 \pi}}\right)}
\right \}, 
\end{align}
and combining the results in the previous subsection, we have proved Theorem \ref{thm:bound_1_refine}.


\section{Proof for Theorem \ref{thm:bound_comblinucb}}
\label{proof:comblinucb}

We start by writing an alternative formula for $\Sigma_t$ and $\bar{\theta}_t$.
Notice that based on Algorithm \ref{algorithm:kalman}, we have:
\begin{align}
\Sigma_t^{-1} =& \,  \frac{1}{\lambda^2} I + \frac{1}{\sigma^2} \sum_{\tau=1}^{t-1} \sum_{k=1}^{\left| A^{\tau} \right|} \phi_{a_k^\tau} \phi_{a_k^\tau}^T  \nonumber \\
\Sigma_t^{-1} \bar{\theta}_t = & \,
\frac{1}{\sigma^2} 
\sum_{\tau=1}^{t-1} \sum_{k=1}^{\left| A^\tau \right|} \phi_{a_k^\tau} \bw_{\tau} \left( a_k^\tau \right) \label{eqn:kalman_alternative_update}
\end{align}
Interested readers might refer to Appendix \ref{sec:kalman_alternative} for the derivation of Equation (\ref{eqn:kalman_alternative_update}).
The proof proceeds as follows. We first construct a confidence set of $\theta^\ast$ based on the ``self normalized bound" developed in \cite{abbasi-yadkori11improved}. Then we derive a regret bound based on Lemma \ref{lem:key_ineq_1} derived above.

\subsection{Confidence Set}
\label{sec:confidence_set}
Our construction of confidence set is motivated by the analysis in \citep{shipra2013}.
We start by defining some useful notation. Specifically, for any $t=1,2,\ldots, n$, any $k=1,2, \ldots, \left| A^t \right|$, we define
\[
\eta_{t,k} = \bw_t \left(a_k^t \right) - \bar{\bw} \left( a_k^t \right).
\]
One key observation is that $\eta_{t,k}$'s form a Martingale difference sequence (MDS)\footnote{Note that the notion of ``time" is indexed by a pair $(t,k)$, and follows the lexicographical order.} since $\bw(e)$'s are statistically independent under $P$. Moreover, since $\bw_t \left(a_k^t \right)$ is bounded in interval $[0,1]$,
$\eta_{t,k}$'s are sub-Gaussian with constant $R=1$. We further define
\begin{align*}
V_t = & \, \frac{\sigma^2}{\lambda^2} I + \sum_{\tau=1}^{t-1} \sum_{k=1}^{\left| A^\tau\right|} \phi_{a_k^\tau} \phi_{a_k^\tau}^T \\
\xi_t = & \, \sum_{\tau=1}^{t-1} \sum_{k=1}^{\left| A^\tau \right|} \phi_{a_k^\tau} \eta_{\tau, k}
\end{align*}
As we will see later, we define $V_t$ and $\xi_t$ to use the ``self normalized bound" developed in \cite{abbasi-yadkori11improved} (see Theorem 1 of \cite{abbasi-yadkori11improved}). Notice that based on the above definition, we have
$\Sigma_t^{-1} = \frac{1}{\sigma^2} V_t$, and
\[
\bar{\theta}_t - \theta^* = \Sigma_t \left( \frac{1}{\sigma^2} \xi_t - \frac{1}{\lambda^2} \theta^* \right).
\]
To see why the second equality holds, notice that
\begin{align*}
\Sigma_t^{-1} \bar{\theta}_t = & \, \frac{1}{\sigma^2} \sum_{\tau=1}^{t-1} \sum_{k=1}^{\left| A^\tau \right|} \phi_{a_k^\tau} \left(
\phi_{a_k^\tau}^T \theta^* + \eta_{\tau,k} \right)\\
= & \,
\left(
\Sigma_t^{-1} -\frac{1}{\lambda^2} I
\right) \theta^* 
+ \frac{1}{\sigma^2} \xi_t .
\end{align*}
Hence, for any $e \in E$, we have
\begin{align*}
\left |
\left < 
\phi_e, \bar{\theta}_t - \theta^*
\right >
\right | = & \,
\left|
\phi_e^T \Sigma_t \left(\frac{1}{\sigma^2} \xi_t - \frac{1}{\lambda^2} \theta^* \right)
\right | \\
\leq & \,
\left \| \phi_e \right \|_{\Sigma_t} \left \| \frac{1}{\sigma^2} \xi_t - \frac{1}{\lambda^2} \theta^* \right \|_{\Sigma_t} \\
\leq & \,
\left \| \phi_e \right \|_{\Sigma_t}
\left[
 \frac{1}{\sigma^2} \left \| \xi_t  \right \|_{\Sigma_t} + \frac{1}{\lambda^2} \left \| \theta^* \right \|_{\Sigma_t}
\right],
\end{align*}
where the first inequality follows from the Cauchy-Schwarz inequality, and the second inequality follows from the triangular inequality.
Notice that
\[
 \left \| \theta^* \right \|_{\Sigma_t} \leq  \left \| \theta^* \right \|_{\Sigma_1} = \lambda \left \| \theta^* \right \|_2,
\]
hence we have
\[
\left |
\left < 
\phi_e, \bar{\theta}_t - \theta^*
\right >
\right | 
\leq
\left \| \phi_e \right \|_{\Sigma_t}
\left[
 \frac{1}{\sigma^2} \left \| \xi_t  \right \|_{\Sigma_t} + \frac{1}{\lambda} \left \| \theta^* \right \|_2
\right].
\]
Moreover, we have
\begin{align*}
\frac{1}{\sigma^2} \left \|
\xi_t
\right \|_{\Sigma_t} =& \, \frac{1}{\sigma^2} \left \|
\xi_t
\right \|_{\sigma^2 V_t^{-1}} = \frac{1}{\sigma} \left \| \xi_t \right \|_{V_t^{-1}}.
\end{align*}
So we have
\begin{equation}
\left |
\left < 
\phi_e, \bar{\theta}_t - \theta^*
\right >
\right | 
\leq
\left \| \phi_e \right \|_{\Sigma_t}
\left[
\frac{1}{\sigma} \left \| \xi_t \right \|_{V_t^{-1}} + \frac{1}{\lambda} \left \| \theta^* \right \|_2
\right].
\end{equation}
The above inequality always holds. We now provide a high probability bound on $\left \| \xi_t \right \|_{V_t^{-1}}$, based on
the ``self normalized bound" proposed in \cite{abbasi-yadkori11improved}.
From Theorem 1 of \cite{abbasi-yadkori11improved}, we know for any $\delta \in (0,1)$, with probability at least $1-\delta$,
\[
\left \| \xi_t \right \|_{V_t^{-1}} \leq \sqrt{2 \log \left( \frac{\det (V_t)^{1/2} \det (V_1)^{-1/2}}{\delta}\right)} \quad \forall t=1,2,\ldots.
\]
Obviously, $\det \left( V_1 \right)=\left[ \frac{\sigma^2}{\lambda^2} \right]^d$, 
on the other hand, we have
\begin{align*}
\left[ \det (V_t) \right]^{1/d} \leq & \, \frac{\trace (V_t)}{d} = \frac{\sigma^2}{\lambda^2} + \frac{1}{d} \sum_{\tau=1}^{t-1} \sum_{k=1}^{\left| A^\tau \right|} \| \phi_{a_k^\tau} \|^2 \leq  \frac{\sigma^2}{\lambda^2} + \frac{(t-1)K}{d} ,
\end{align*}
where the last inequality follows from the assumption that $\| \phi_e \| \leq 1$. Hence, for $t \leq n$, we have
\begin{align*}
\left[ \det (V_t) \right]^{1/d} \leq &   \frac{\sigma^2}{\lambda^2} + \frac{nK}{d} .
\end{align*}
Thus, with probability at least $1-\delta$, we have
\[
\left \| \xi_t \right \|_{V_t^{-1}} \leq \sqrt{d \log \left( 1+ \frac{n K \lambda^2}{d \sigma^2}\right) + 2 \log \left( \frac{1}{\delta} \right)}\quad \forall t=1,2,\ldots, n .
\]
 
Thus, we have the following lemma:
\begin{lemma}
For any $\lambda, \sigma>0$ and any $\delta \in (0,1)$, with probability at least $1-\delta$, we have
\begin{equation}
\left |
\left < 
\phi_e, \bar{\theta}_t - \theta^*
\right >
\right | 
\leq
\left \| \phi_e \right \|_{\Sigma_t}
\left[
\frac{1}{\sigma} \sqrt{d \log \left( 1+ \frac{n K \lambda^2}{d \sigma^2}\right) + 2 \log \left( \frac{1}{\delta} \right)} + \frac{\left \| \theta^* \right \|_2 }{\lambda} 
\right],
\end{equation}
for all $t=1,2,\ldots, n$, and for all $e \in E$.
\end{lemma}

Notice that $\left \| \phi_e \right \|_{\Sigma_t} = \sqrt{\phi_e^T \Sigma_t \phi_e}$, thus, the above lemma immediately implies the following lemma:
\begin{lemma}
\label{lemma:confidence}
For any $\lambda, \sigma>0$, any $\delta \in (0,1)$, and any
\[
c \geq \frac{1}{\sigma} \sqrt{d \log \left( 1+ \frac{n K \lambda^2}{d \sigma^2}\right) + 2 \log \left( \frac{1}{\delta} \right)} + \frac{\left \| \theta^* \right \|_2}{\lambda} ,
\]
with probability at least $1-\delta$, we have
\[
\left< \phi_e, \bar{\theta}_t \right> - c \sqrt{\phi_e^T \Sigma_t \phi_e} \leq 
\left< \phi_e, \theta^* \right> \leq \left< \phi_e, \bar{\theta}_t \right> + c \sqrt{\phi_e^T \Sigma_t \phi_e},
\]
for all $e \in E$ and $t=1,2,\ldots n$.
\end{lemma}

Notice that 
\[
\left< \phi_e, \theta^* \right> \leq \left< \phi_e, \bar{\theta}_t \right> + c \sqrt{\phi_e^T \Sigma_t \phi_e}
\]
is exactly $\bar{\bw} (e) \leq \hat{\bw}_t (e)$.

\subsection{Regret Analysis}
We define event $G$ as
\begin{align}
G= \left \{
\left< \phi_e, \bar{\theta}_t \right> - c \sqrt{\phi_e^T \Sigma_t \phi_e} \leq 
\left< \phi_e, \theta^* \right> \leq \left< \phi_e, \bar{\theta}_t \right> + c \sqrt{\phi_e^T \Sigma_t \phi_e} \;
\forall e \in E, \, \forall t=1,\ldots, n
\right \}, \label{eqn:G}
\end{align}
and use $\bar{G}$ to denote the complement of event $G$.
Recall that Lemma \ref{lemma:confidence} states that if
\begin{equation}
\label{eq:c_bound}
c \geq \frac{1}{\sigma} \sqrt{d \log \left( 1+ \frac{n K \lambda^2}{d \sigma^2}\right) + 2 \log \left( \frac{1}{\delta} \right)} + \frac{1}{\lambda} \left \| \theta^* \right \|_2,
\end{equation}
then $\mathbb{P} (G) \geq 1-\delta$. Moreover, by definition, under event $G$, we have $\bar{\bw} (e) \leq \hat{\bw}_t (e)$,
for all $t=1,\ldots, n$ and any $e \in E$.

Notice that
\begin{align*}
R(n) = & \, \sum_{t=1}^n \mathbb{E} \left[ \sum_{e \in A^*} \bw_t (e) -  \sum_{e \in A^t} \bw_t (e) \right] \\
= & \, \sum_{t=1}^n \mathbb{E} \left[ \sum_{e \in A^*} \bar{\bw} (e) -  \sum_{e \in A^t} \bar{\bw} (e) \right] \\
= & \, \mathbb{P} \left(G \right) \sum_{t=1}^n \mathbb{E} \left[ \sum_{e \in A^*} \bar{\bw} (e) -  \sum_{e \in A^t} \bar{\bw} (e) \middle | G \right]
+ \mathbb{P} \left ( \bar{G} \right)
\sum_{t=1}^n \mathbb{E} \left[ \sum_{e \in A^*} \bar{\bw} (e) -  \sum_{e \in A^t} \bar{\bw} (e) \middle | \bar{G} \right] \\
\leq & \,
\sum_{t=1}^n \mathbb{E} \left[ \sum_{e \in A^*} \bar{\bw} (e) -  \sum_{e \in A^t} \bar{\bw} (e) \middle | G \right] + 
\mathbb{P} \left ( \bar{G} \right) nK,
\end{align*}
where the last inequality follows from the naive bound on the realized regret. If $c$ satisfies inequality (\ref{eq:c_bound}), we have
$\mathbb{P} \left( \bar{G} \right) \leq \delta$, hence we have
\[
R(n) \leq \sum_{t=1}^n \mathbb{E} \left[ \sum_{e \in A^*} \bar{\bw} (e) -  \sum_{e \in A^t} \bar{\bw} (e) \middle | G \right] + 
 nK \delta.
\]
Finally, we bound $\sum_{t=1}^n \mathbb{E} \left[ \sum_{e \in A^*} \bar{\bw} (e) -  \sum_{e \in A^t} \bar{\bw} (e) \middle | G \right] $
using a worst-case bound conditioning on $G$ (worst-case over all the possible random realizations), notice that conditioning on
$G$, we have
\begin{align*}
\sum_{e \in A^*} \bar{\bw} (e) \leq \sum_{e \in A^*} \hat{\bw}_t (e) \leq \sum_{e \in A^t} \hat{\bw}_t (e),
\end{align*}
where the first inequality follows from the definition of event $G$, and the second inequality follows from that 
$A^t$ is the \emph{exact} solution of the combinatorial optimization problem $(E, \cA, \hat{\bw}_t)$.
Thus we have
\begin{align*}
\sum_{e \in A^*} \bar{\bw} (e) - \sum_{e \in A^t} \bar{\bw} (e) \leq & \,
\sum_{e \in A^t} \hat{\bw}_t (e) - \sum_{e \in A^t} \bar{\bw} (e) \\
=& \,
\sum_{e \in A^t} \left[
\left < 
\phi_e, \bar{\theta}_t - \theta^*
\right > + c\sqrt{\phi_e^T \Sigma_t \phi_e}
\right] \\
\leq & \,
2c \sum_{e \in A^t} \sqrt{\phi_e^T \Sigma_t \phi_e},
\end{align*}
where the last inequality follows from the definition of $G$.
Recall that from Lemma \ref{lem:key_ineq_1}, we have
\[
\sum_{t=1}^n \sum_{e \in A^t} \sqrt{\phi_e^T \Sigma_t \phi_e} \leq K \lambda\sqrt{\frac{d n  \log \left( 1 + \frac{nK \lambda^2}{d \sigma^2} \right)}{\log\left(1+ \frac{\lambda^2}{\sigma^2} \right)}}.\]
Thus we have
\[
\sum_{t=1}^n \mathbb{E} \left[ \sum_{e \in A^*} \bar{\bw} (e) -  \sum_{e \in A^t} \bar{\bw} (e) \middle | G \right] 
\leq 2c \mathbb{E} \left[ \sum_{t=1}^n \sum_{e \in A^t} \sqrt{\phi_e^T \Sigma_t \phi_e}\right]
\leq 2c  K \lambda\sqrt{\frac{d n  \log \left( 1 + \frac{nK \lambda^2}{d \sigma^2} \right)}{\log\left(1+ \frac{\lambda^2}{\sigma^2} \right)}},
\]
which implies 
\[
R(n) \leq 2c  K \lambda\sqrt{\frac{d n  \log \left( 1 + \frac{nK \lambda^2}{d \sigma^2} \right)}{\log\left(1+ \frac{\lambda^2}{\sigma^2} \right)}} + 
 nK \delta.
\]


\section{Technical Lemma}
\label{sec:kalman_alternative}
In this section, we derive Equation (\ref{eqn:kalman_alternative_update}). We first prove the following technical lemma:
\begin{lemma}
\label{lem:tech_kalman}
For any $\phi , \bar{\theta} \in \realset^d$, any positive definite $\Sigma \in \realset^{d \times d}$, any $\sigma>0$, and 
any $w \in \realset$, if we define
\begin{align}
\Sigma_{\mathrm{new}} = & \,  \Sigma - \frac{\Sigma \phi \phi^T \Sigma}{\phi^T \Sigma \phi +\sigma^2} 
\nonumber \\
\bar{\theta}_{\mathrm{new}} = & \, \left[ I -\frac{\Sigma \phi \phi^T}{\phi^T \Sigma \phi +\sigma^2} \right] \bar{\theta}+
\left[
\frac{\Sigma \phi }{\phi^T \Sigma \phi +\sigma^2} 
\right] w, \nonumber
\end{align}
then we have
\begin{align}
\Sigma_{\mathrm{new}}^{-1} =& \, \Sigma^{-1} + \frac{1}{\sigma^2} \phi \phi^T  \label{eqn:tech:cov} \\
\Sigma_{\mathrm{new}}^{-1} \bar{\theta}_{\mathrm{new}} =& \, \Sigma^{-1}\bar{\theta} + \frac{1}{\sigma^2} \phi w \label{eqn:tech:mean}.
\end{align}
\end{lemma}
\proof
Notice that Equation (\ref{eqn:tech:cov}) follows directly from the Woodbury matrix identity (matrix inversion lemma). 
We now prove Equation (\ref{eqn:tech:mean}). Notice that we have
\begin{align}
\bar{\theta}_{\mathrm{new}} = & \, \left[ I -\frac{\Sigma \phi \phi^T}{\phi^T \Sigma \phi +\sigma^2} \right] \bar{\theta}+
\left[
\frac{\Sigma \phi }{\phi^T \Sigma \phi +\sigma^2} 
\right] w \nonumber \\
=& \, \left[ \Sigma -\frac{\Sigma \phi \phi^T \Sigma}{\phi^T \Sigma \phi +\sigma^2} \right] \Sigma^{-1} \bar{\theta}+
\left[
\frac{\Sigma \phi }{\phi^T \Sigma \phi +\sigma^2} 
\right] w \nonumber \\
=& \,
\Sigma_{\mathrm{new}} \Sigma^{-1} \bar{\theta}+
\left[
\frac{\Sigma \phi }{\phi^T \Sigma \phi +\sigma^2} 
\right] w, \nonumber
\end{align}
that is,
\begin{align}
\Sigma_{\mathrm{new}}^{-1} \bar{\theta}_{\mathrm{new}} =
\Sigma^{-1} \bar{\theta}+
\left[
\frac{\Sigma_{\mathrm{new}}^{-1} \Sigma \phi }{\phi^T \Sigma \phi +\sigma^2} 
\right] w. \label{eqn:tech:p1}
\end{align}
Notice that
\begin{align}
\Sigma_{\mathrm{new}}^{-1} \Sigma \phi = & \, \left[ \Sigma^{-1} + \frac{1}{\sigma^2} \phi \phi^T \right] \Sigma \phi
=\phi +\frac{\phi^T \Sigma \phi}{\sigma^2} \phi=
\frac{\sigma^2+ \phi^T \Sigma \phi}{\sigma^2} \phi.
\label{eqn:tech:p2}
\end{align}
Plug Equation (\ref{eqn:tech:p2}) into Equation (\ref{eqn:tech:p1}), we have 
Equation (\ref{eqn:tech:mean}).
\endproof

Based on Lemma \ref{lem:tech_kalman}, by mathematical induction, we have
\begin{align}
\Sigma_t^{-1} =& \,  \Sigma_1^{-1} + \frac{1}{\sigma^2} \sum_{\tau=1}^{t-1} \sum_{k=1}^{\left| A^{\tau} \right|} \phi_{a_k^\tau} \phi_{a_k^\tau}^T  \nonumber \\
\Sigma_t^{-1} \bar{\theta}_t = & \, \Sigma_1^{-1} \bar{\theta}_1
\frac{1}{\sigma^2} 
\sum_{\tau=1}^{t-1} \sum_{k=1}^{\left| A^\tau \right|} \phi_{a_k^\tau} \bw_{\tau} \left( a_k^\tau \right), \nonumber
\end{align}
further noting that $\Sigma_1 = \lambda^2 I$ and $\bar{\theta}_1=0$, we can derive Equation (\ref{eqn:kalman_alternative_update}).

\section{A Variant of Theorem \ref{thm:bound_comblinucb} for Approximation Algorithms}
\label{appendix:variant}
By suitably redefining the realized regret, we can prove a variant of Theorem \ref{thm:bound_comblinucb}
in which $\oracle$ can be an approximation algorithm. Specifically, for a (possibly approximation) algorithm $\oracle$, let
$A^*(\bw)$ be the solution of $\oracle$ to the optimization problem $(E, \cA , \bw)$, we say $\gamma \in [0,1)$ is a \emph{sub-optimality gap} of
$\oracle$ if
\begin{align}
f(A^*(\bw), \bw) \geq (1- \gamma) \max_{A \in \cA} f(A, \bw), \quad \forall \bw.   \label{eqn:sub_optimality_gap}
\end{align}
Then we define the (scaled) realized regret $R^{\gamma}_t$ as
\begin{align}
R^{\gamma}_t = f \left( A^{\mathrm{opt}}, \bw_t \right)- \frac{ f \left( A^{t}, \bw_t \right)}{1-\gamma},
\end{align}
where $A^{\mathrm{opt}}$ is the exact solution to the optimization problem $(E, \cA , \bar{\bw})$.
The (scaled) cumulative regret $R^\gamma(n)$ is defined as 
\[
R^\gamma(n)=\sum_{t=1}^n \mathbb{E} \left[ R_t^{\gamma} \middle | \bar{\bw} \right].
\]
Under the assumptions that 
(1) the support of $P$ is a subset of 
$[0,1]^L$ (i.e. $\bw_t (e) \in [0,1]$ $\forall t$ and $\forall e \in E$), 
(2) the item weight $\bw(e)$'s are statistically independent under $P$,
and
(3) the oracle $\oracle$ has sub-optimality gap $\gamma \in [0,1)$,
we have the following variant of Theorem \ref{thm:bound_comblinucb}
when $\comblinucb$ is applied to coherent learning cases:

\begin{theorem}
\label{thm:bound_comblinucb_variant}
For any $\lambda, \sigma>0$, any $\delta \in (0,1)$, and any $c$ satisfying
\begin{align}
c \geq \frac{1}{\sigma} \sqrt{d \ln \left( 1+ \frac{n K \lambda^2}{d \sigma^2}\right) + 2 \ln \left( \frac{1}{\delta} \right)} + \frac{\left \| \theta^* \right \|_2}{\lambda} ,
\label{eq:c_lb_variant}
\end{align}
if $\bar{\bw} = \Phi \theta^\ast$ and the above two assumptions hold, then 
under $\comblinucb$ algorithm with parameter $(\Phi, \lambda, \sigma, c)$, we have
\[
R^{\gamma}(n) \leq \frac{2c  K \lambda}{1-\gamma} \sqrt{\frac{d n  \ln \left( 1 + \frac{nK \lambda^2}{d \sigma^2} \right)}{\ln \left(1+ \frac{\lambda^2}{\sigma^2} \right)}} + 
 nK \delta.
\]
\end{theorem}

\proof
Notice that Lemma~\ref{lemma:confidence} in Section~\ref{sec:confidence_set} still holds. With
$G$ defined in Equation~(\ref{eqn:G}), we have
\begin{align*}
R^{\gamma}(n) = & \, \sum_{t=1}^n \mathbb{E} \left[ \sum_{e \in A^{\mathrm{opt}}} \bw_t (e) -  \frac{1}{1-\gamma}\sum_{e \in A^t} \bw_t (e) \right] \\
= & \, \sum_{t=1}^n \mathbb{E} \left[ \sum_{e \in A^{\mathrm{opt}}} \bar{\bw} (e) -  \frac{1}{1-\gamma} \sum_{e \in A^t} \bar{\bw} (e) \right] \\
= & \, \mathbb{P} \left(G \right) \sum_{t=1}^n \mathbb{E} \left[ \sum_{e \in A^{\mathrm{opt}}} \bar{\bw} (e) -   \frac{1}{1-\gamma} \sum_{e \in A^t} \bar{\bw} (e) \middle | G \right]
+ \mathbb{P} \left ( \bar{G} \right)
\sum_{t=1}^n \mathbb{E} \left[ \sum_{e \in A^{\mathrm{opt}}} \bar{\bw} (e) -  \frac{1}{1-\gamma}\sum_{e \in A^t} \bar{\bw} (e) \middle | \bar{G} \right] \\
\leq & \,
\sum_{t=1}^n \mathbb{E} \left[ \sum_{e \in A^{\mathrm{opt}}} \bar{\bw} (e) -   \frac{1}{1-\gamma} \sum_{e \in A^t} \bar{\bw} (e) \middle | G \right] + 
\mathbb{P} \left ( \bar{G} \right) nK,
\end{align*}
where the last inequality follows from the naive bound on $R_t^{\gamma}$. If $c$ satisfies inequality (\ref{eq:c_bound}), we have
$\mathbb{P} \left( \bar{G} \right) \leq \delta$, hence we have
\[
R^{\gamma}(n) \leq \sum_{t=1}^n \mathbb{E} \left[ \sum_{e \in A^{\mathrm{opt}}} \bar{\bw} (e) -  \frac{1}{1-\gamma}\sum_{e \in A^t} \bar{\bw} (e) \middle | G \right] + 
 nK \delta.
\]
Finally, we bound $\sum_{t=1}^n \mathbb{E} \left[ \sum_{e \in A^{\mathrm{opt}}} \bar{\bw} (e) -  \frac{1}{1-\gamma}\sum_{e \in A^t} \bar{\bw} (e) \middle | G \right] $
using a worst-case bound conditioning on $G$ (worst-case over all the possible random realizations), notice that conditioning on
$G$, we have
\begin{align*}
\sum_{e \in A^{\mathrm{opt}}} \bar{\bw} (e) \leq \sum_{e \in A^{\mathrm{opt}}} \hat{\bw}_t (e) \leq 
\max_{A \in \cA} \sum_{e \in A} \hat{\bw}_t (e) \leq
\frac{1}{1-\gamma}\sum_{e \in A^t} \hat{\bw}_t (e),
\end{align*}
where 
\begin{itemize}
\item The first inequality follows from the definition of event $G$. Specifically, under event $G$, $\bar{\bw} (e) \leq \hat{\bw}_t (e)$ for all $t=1,\ldots, n$ and
all $e \in E$.
\item The second inequality follows from $A^{\mathrm{opt}} \in \cA$.
\item The last inequality follows from $A^t \leftarrow \oracle (E, \cA, \hat{\bw}_t)$ and $\oracle$ has sub-optimality gap $\gamma$ (see Equation (\ref{eqn:sub_optimality_gap})). 
\end{itemize}
Thus we have
\begin{align*}
\sum_{e \in A^{\mathrm{opt}}} \bar{\bw} (e) - \frac{1}{1-\gamma}\sum_{e \in A^t} \bar{\bw} (e) \leq & \,
\frac{1}{1-\gamma} \left[ \sum_{e \in A^t} \hat{\bw}_t (e) - \sum_{e \in A^t} \bar{\bw} (e) \right]\\
=& \,
\frac{1}{1-\gamma} \sum_{e \in A^t} \left[
\left < 
\phi_e, \bar{\theta}_t - \theta^*
\right > + c\sqrt{\phi_e^T \Sigma_t \phi_e}
\right] \\
\leq & \,
\frac{2c}{1-\gamma} \sum_{e \in A^t} \sqrt{\phi_e^T \Sigma_t \phi_e},
\end{align*}
where the last inequality follows from the definition of $G$.
Recall that from Lemma \ref{lem:key_ineq_1}, we also have
\[
\sum_{t=1}^n \sum_{e \in A^t} \sqrt{\phi_e^T \Sigma_t \phi_e} \leq K \lambda\sqrt{\frac{d n  \log \left( 1 + \frac{nK \lambda^2}{d \sigma^2} \right)}{\log\left(1+ \frac{\lambda^2}{\sigma^2} \right)}}.\]
Putting the above inequalities together, we have proved the theorem.
\endproof

\end{document}